\documentclass[nohyperref]{article}

\usepackage{microtype}
\usepackage{graphicx}
\usepackage{subfigure}
\usepackage{booktabs} % for professional tables
\usepackage{algorithm}
\usepackage{algcompatible}
\usepackage{algpseudocode}
\usepackage{dsfont}

\usepackage{hyperref}

\def\*#1{\mathbf{#1}}

\usepackage[accepted]{icml2022_arxiv}

\usepackage{amsmath}
\usepackage{amssymb}
\usepackage{mathtools}
\usepackage{amsthm}
\usepackage{bm}

\usepackage{multirow}

\usepackage{float}% http://ctan.org/pkg/float
\floatstyle{plain}
\restylefloat{figure}% ...figure environment contents.

\usepackage[capitalize,noabbrev]{cleveref}

\theoremstyle{plain}
\newtheorem{theorem}{Theorem}[section]
\newtheorem{proposition}[theorem]{Proposition}
\newtheorem{lemma}[theorem]{Lemma}
\newtheorem{corollary}[theorem]{Corollary}
\theoremstyle{definition}

\newtheorem{assumption}[theorem]{Assumption}
\theoremstyle{remark}

\usepackage[textsize=tiny]{todonotes}

\def\H{\mathcal{H}}

\def\G{\mathcal{G}}

\def\X{\mathcal{X}}

\def\Y{\mathcal{Y}}

\def\L{\mathcal{L}}

\def \d1{\mathds{1}}

\def\E{\mathbb{E}}

\def\1{\mathbf{1}}
\def\P{\mathbb{P}}
\def\R{\mathbb{R}}

\DeclareMathOperator*{\argmin}{arg\,min}

\def\t{\top}

\def\argmin{\text{argmin}}
\def\max{\text{max}}

\newcommand{\sign}{\mathrm{sign}}

\newcommand{\ignore}[1]{}

\def\Pin{\mathbb{P}_{\text{in}}}
\def\Pwild{\mathbb{P}_{\text{wild}}}

\icmltitlerunning{Training OOD Detectors in their Natural Habitats}

\begin{document}

\twocolumn[
\icmltitle{Training OOD Detectors in their Natural Habitats}



\icmlsetsymbol{equal}{*}

\begin{icmlauthorlist}
\icmlauthor{Julian Katz-Samuels}{equal,yyy}
\icmlauthor{Julia Nakhleh}{equal,comp}
\icmlauthor{Robert Nowak}{xxx}
\icmlauthor{Yixuan Li}{comp}
\end{icmlauthorlist}

\icmlaffiliation{yyy}{Institute for Foundations of Data Science, University of Wisconsin, Madison}
\icmlaffiliation{comp}{Department of Computer Sciences, University of Wisconsin, Madison}
\icmlaffiliation{xxx}{Department of Electrical and Computer Engineering, University of Wisconsin, Madison}

\icmlcorrespondingauthor{Julian Katz-Samuels}{katzsamuels@wisc.edu}

\icmlkeywords{Machine Learning, ICML}

\vskip 0.3in
]



\printAffiliationsAndNotice{\icmlEqualContribution} % otherwise use the standard text.

\begin{abstract}
Out-of-distribution (OOD) detection is important for machine learning models deployed in the wild. Recent methods use auxiliary outlier data to regularize the model for improved OOD detection. However, these approaches make a strong distributional assumption that the auxiliary outlier data is completely separable from the in-distribution (ID) data. In this paper, we propose a novel framework that leverages wild mixture data---that naturally consists of both ID and OOD samples. Such wild data is abundant and arises freely upon deploying a machine learning classifier in their \emph{natural habitats}. Our key idea is to formulate a constrained optimization problem and to show how to tractably solve it. 
Our learning objective maximizes the OOD detection rate, subject to constraints on the classification error of ID data and on the OOD error rate of ID examples. We extensively evaluate our approach on common OOD detection tasks and demonstrate superior performance. Code is available at \url{https://github.com/jkatzsam/woods_ood}.
\end{abstract}

\section{Introduction}

Out-of-distribution (OOD) detection has become a central challenge in safely deploying machine learning models in the wild, where test-time data can naturally arise from a mixture distribution of both knowns and unknowns~\cite{openworld}. 
Concerningly, modern neural networks are shown to produce overconfident and therefore untrustworthy predictions for unknown OOD inputs~\cite{fool}. 
To mitigate the issue, recent works have explored training with an auxiliary outlier dataset, where the model is regularized to produce lower confidence~\cite{hendrycks2018deep} or higher energies~\cite{liu2020energy} on the outlier data. These methods have demonstrated encouraging OOD detection performance over the counterpart without auxiliary data. 

Despite this promise, there are two primary limitations to the existing methods. First, the auxiliary data distribution collected offline may not be a good match for the true distribution of unknown data in the wild, thus the learned model may fail to detect deployment-time OOD data. Second, collecting such data can be very labor-intensive and inflexible, and necessitates careful data cleaning to ensure the auxiliary outlier data does not overlap with the ID data.
We address these challenges by leveraging unlabeled ``in-the-wild'' data --- which can be collected almost \emph{for free} upon deploying a machine learning classifier in the open world, and has been largely overlooked for OOD learning purposes. Such data is available in abundance, does not require any human annotation, and is often a much better match to the true test time distribution than data collected offline.
While this setting naturally suits many real-world applications, it also poses unique challenges since the wild data distribution is not pure and consists of both ID data and OOD data. 
\begin{center}
\fbox{\begin{minipage}{20em}
In this paper, we propose a novel framework that enables effectively exploiting unlabeled in-the-wild data for OOD detection. Unlabeled wild data is frequently available since it is produced essentially whenever deploying an existing classifier in a real-world system. This setting can be viewed as training OOD detectors in their \emph{natural habitats}.
\end{minipage}}
\end{center}

Our learning framework revolves around building the OOD classifier using only labeled ID data from $\P_\text{in}$ and unlabeled wild data from $\P_\text{wild}$, which can be considered to be a mixture of $\P_\text{in}$ and an unknown (OOD) distribution. To deal with the lack of a ``clean" set of OOD examples, our key idea is to formulate a constrained optimization problem. In a nutshell, our learning objective aims to minimize the error of classifying data from $\P_\text{wild}$ as ID, subject to two constraints: \emph{(i)} the error of declaring an ID data from $\P_\text{in}$  as OOD must be low, and \emph{(ii)} the multi-class classification model must maintain the best-achievable accuracy (or close to it) of a baseline classifier designed without an OOD detection requirement. Even though our framework does not have access to a ``clean" OOD dataset, we show both empirically and theoretically that it can learn to accurately detect OOD examples. 
Our work is inspired by the semi-supervised novelty detection (SSND) framework~\cite{blanchard2010semi}, yet differs by simultaneously considering both classification and OOD detection tasks.

\begin{figure*}
    \centering
    \includegraphics[width=\textwidth]{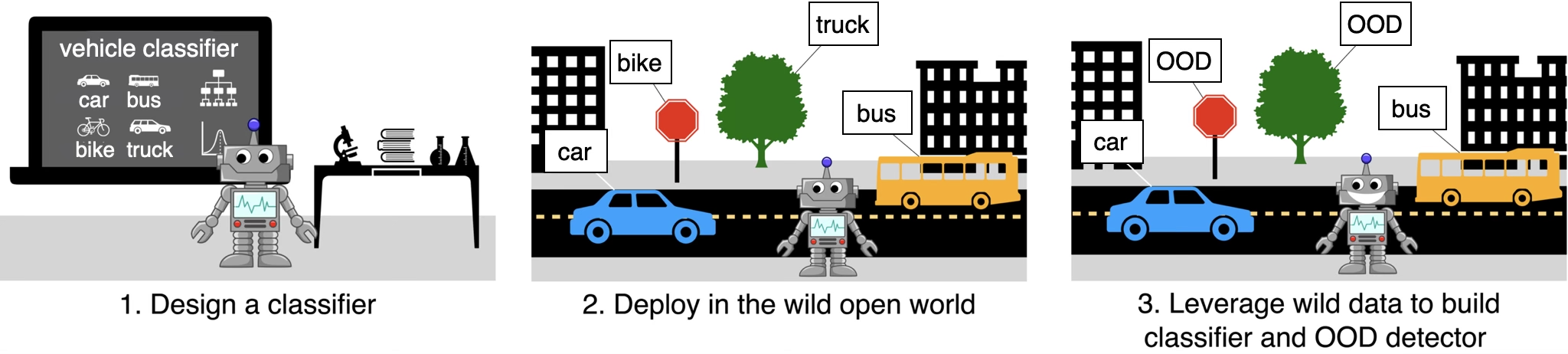}
    \caption{Overview of our learning framework. A model pre-trained to perform in-distribution (ID) classification can be deployed in the ``wild'' open world, where it will encounter large quantities of unlabeled ID \textit{and} OOD data. Using our WOODS approach, the model can be fine-tuned using the ``wild'' data to perform accurate OOD detection as well as ID classification.}
    \label{fig:woods_overview}
\end{figure*}

Beyond the mathematical framework, a key contribution of our paper is a constrained optimization implementation of the framework for deep neural networks.  We propose a novel training procedure based on the augmented Lagrangian method, or ALM~\cite{hestenes1969multiplier}. While ALM is an established approach to optimization with functional constraints, its adaptation to modern deep learning is not straightforward or common. In particular, we adapt ALM to our problem setting with inequality constraints, obtaining an end-to-end training algorithm using stochastic gradient descent. Unlike methods that add a regularization term to the training objective, our method via constrained optimization offers strong guarantees~(\emph{c.f.} Section~\ref{sec:objective}). 

We extensively evaluate our approach on common OOD detection tasks and establish state-of-the-art performance. For completeness, we compare with two groups of approaches: (1) trained with only $\P_\text{in}$ data, and (2) trained with both $\P_\text{in}$ and an auxiliary dataset. On CIFAR-100, compared to a strong baseline using only $\Pin$, our method outperforms by \textbf{48.10}\% (FPR95) on average. The performance gain precisely demonstrates the advantage of incorporating unlabeled wild data. Our method also outperforms Outlier Exposure (\texttt{OE})~\cite{hendrycks2018deep} by \textbf{7.36}\% in FPR95. Our {key contributions} are summarized as follows:
\begin{itemize}
% \vspace{-0.2cm}
\item We propose a novel OOD detection framework via constrained optimization (dubbed  \texttt{WOODS}, Wild OOD detection sans-Supervision), capable of exploiting unlabeled wild data. We show how to integrate constrained optimization into modern deep neural networks and solve it tractably.
% \vspace{-0.2cm}
\item   We provide novel theoretical insights that support \texttt{WOODS}, in particular for the choice of loss functions.
% \vspace{-0.2cm}
\item We perform extensive ablations and comparisons under: (1) a diverse range of datasets, (2) different mixture ratios $\pi$ of ID and OOD in $\Pwild$, and (3) different assumptions on the relationship between the wild distribution $\P_\text{wild}$ and the test-time distribution. Our method establishes \emph{state-of-the-art} results, significantly outperforming existing methods.
\end{itemize}

\section{Problem Setup}\label{sec:problem_setup}

\paragraph{Labeled In-distribution Data} 
Let $\X=\R^d$ denote the input space and $\Y=\{1,\ldots, K\}$ denote the label space. We assume access to the labeled training set $\mathcal{D}_{\text{in}}^{\text{train}} = \{(\*x_i, y_i)\}_{i=1}^n$, drawn \emph{i.i.d.} from the joint data distribution $\P_{\mathcal{X}\mathcal{Y}}$. Let $\P_{\text{in}}$ denote the marginal distribution on $\mathcal{X}$, which is also referred to as the \emph{in-distribution} (ID). Let $f_\theta: \X \mapsto \mathbb{R}^{|\mathcal{Y}|}$  denote a function for the classification task, which predicts the label of an input sample.

\paragraph{Out-of-distribution Detection} When deploying machine learning models in the real world, a reliable classifier should not only accurately classify known ID samples, but also identify as ``unknown'' any \emph{out-of-distribution} (OOD) input---samples from a different distribution $\P_\text{out}^\text{test}$ that the model has not been exposed to during training. This can be achieved through having an OOD classifier, in addition to the multi-class classifier $f_\theta$. Samples detected as OOD will be rejected; samples detected as ID will be classified by $f_\theta$.

OOD detection can be formulated as a binary classification problem. At test time, the goal is to decide whether a test-time input $\*x \in \X$ is from the in-distribution $\P_\text{in}$ (ID) or not (OOD). We denote  $g_\theta : \X \mapsto \{\text{in}, \text{out} \}$ as the function mapping for OOD detection.

\paragraph{Unlabeled in-the-wild Data} 
A major challenge in OOD detection is the lack of labeled examples of OOD. In particular, the sample space for potential OOD data can be prohibitively large, making it expensive to collect labeled OOD data. In this paper, we incorporate unlabeled in-the-wild samples $\tilde{\*x}_{1},\ldots, \tilde{\*x}_{m}$ into OOD detection. These samples consist of potentially both ID and OOD data, and can be collected almost for free upon deploying an existing classifier (say $f_\theta$) in its natural habitats. 
We use the Huber contamination model~\cite{huber1964} to model the marginal distribution of the wild data:
$$
\P_\text{wild} := (1-\pi) \P_\text{in} + \pi \P_\text{out},
$$
where $\pi \in (0,1]$. We note that the case $\pi = 0$ is technically possible in controlled deployment environments, but it is straightforward since no novelties occur and we therefore do not consider it.

\paragraph{Goal:} Our learning framework revolves around building the OOD classifier $g_\theta$ and multi-class classifier $f_\theta$ by leveraging data from both $\P_\text{in}$ and $\P_\text{wild}$. We use the shared parameters $\theta$ to indicate the fact that they may share the neural network parameters. In testing, we measure the following errors:
\begin{align*}
	& \downarrow \text{FPR}(g_\theta):=\mathbb{E}_{\*x \sim \P_\text{out}^\text{test}}(\d1{\{g_\theta(\*x)=\text{in}\}}),\\
	& \uparrow \text{TPR}(g_\theta):=\mathbb{E}_{\*x \sim \Pin}(\d1{\{g_\theta(\*x)=\text{in} \}}),\\
	&\uparrow\text{Acc}(f_\theta):=\mathbb{E}_{(\*x, y) \sim \P_{\mathcal{X}\mathcal{Y}}}(\d1{\{f_\theta(\*x)=y\}}),
	\end{align*}
where $\d1\{\cdot\}$ is the indicator function and the arrows indicate higher/lower is better. We distinguish between the OOD training distribution $\mathbb{P}_\text{out}$ and the OOD test distribution $\mathbb{P}_\text{out}^\text{test}$ because as is common in out-of-distribution detection, the data available for training may be different from the data at test time. We call the case $\P_\text{out}^\text{test} = \P_\text{out}$  the \emph{stationary wild setting}  and the case with $\P_\text{out}^\text{test} \neq \P_\text{out}$ the \emph{non-stationary wild setting} and will consider both  later in Section~\ref{experiments}.

\section{Method: Out-of-distribution Learning via Constrained Optimization}
\label{sec:method}

In this section, we present a novel framework that performs {out-of-distribution learning with unlabeled data in the wild}. Our framework offers substantial advantages over the counterpart approaches that rely only on the ID data, and naturally suits many applications where machine learning models are deployed in the open world.  

To exploit the in-the-wild data, our key idea is to formulate a constrained optimization problem (Section~\ref{sec:objective}). Moreover, we show how to integrate this constrained optimization problem into modern neural networks and solve it tractably using the Augmented Lagrangian Method (Section~\ref{sec:optimization}).

\subsection{Learning Objective}
\label{sec:objective}

In a nutshell, we formulate the learning objective by aiming to minimize the error of classifying data from $\P_\text{out}$ as ID, subject to \emph{(i)} the error of declaring an ID as OOD is at most $\alpha$, and \emph{(ii)} the multi-class classification model meets some error threshold $\tau$. Mathematically, this can be formalized as a constrained optimization problem:

\paragraph{Objective Overview} Given $\alpha, \tau \in [0,1]$, we aim to optimize:
  \begin{align}
        \inf_\theta & ~ \mathbb{E}_{\*x \sim \P_\text{out}}(\d1{\{g_\theta(\*x)=\text{in}\}}) \label{eq:goal_optimization_problem} \\
        \text{s.t. } &  \mathbb{E}_{\*x \sim \Pin}(\d1{\{g_\theta(\*x)=\text{out}\}}) \leq \alpha \nonumber \\
        & \mathbb{E}_{(\*x, y) \sim \P_{\mathcal{X}\mathcal{Y}}}(\d1{\{f_\theta(\*x)\ne y\}}) \leq \tau. \nonumber
    \end{align}
    
However, we never observe a clean dataset from $\P_\text{out}$, making it difficult to directly solve \eqref{eq:goal_optimization_problem}. To sidestep this issue, we reformulate the learning objective as follows: 
    \begin{align}
        \inf_\theta & ~ \mathbb{E}_{\*x \sim \Pwild}(\d1{\{g_\theta(\*x)=\text{in}\}}) \label{eq:infinite_sample_objective} \\
        \text{s.t. } &  \mathbb{E}_{\*x \sim \Pin}(\d1{\{g_\theta(\*x)=\text{out}\}}) \leq \alpha \nonumber \\
        & \mathbb{E}_{(\*x, y) \sim \P_{\mathcal{X}\mathcal{Y}}}(\d1{\{f_\theta(\*x)\ne y\}}) \leq \tau. \nonumber
    \end{align}
where we replaced the OOD distribution $\P_\text{out}$ in the objective with $\Pwild$, a distribution that we observe an \emph{i.i.d.} dataset from. As we explain in more detail shortly, optimizing \eqref{eq:infinite_sample_objective} is sufficient under mild conditions to solve the original optimization problem \eqref{eq:goal_optimization_problem}.

Empirically, we can solve the optimization problem \eqref{eq:infinite_sample_objective} by minimizing the number of samples $\tilde{\*x}_{1},\ldots, \tilde{\*x}_{m}$ from the wild distribution $\Pwild$ that are labeled as ID, subject to \emph{(i)} labeling at least $1-\alpha$ of the ID samples $\*x_1 \ldots, \*x_n$ correctly, and  \emph{(ii)} achieving the classification performance threshold. Equivalently, we consider solving:
\begin{align}
             \inf_{\theta} & \frac{1}{m} \sum_{i=1}^m \d1 \{ g_{\theta}(\tilde{\*x}_i) = \text{in} \} \label{eq:zero_one_loss_opt} \\
        \text{s.t. } &  \frac{1}{n} \sum_{i=1}^n \d1 \{g_{\theta}(\*x_i) = \text{out} \} \leq \alpha   \nonumber \\
        & \frac{1}{n} \sum_{i=1}^n \d1 \{ f_{\theta}(\*x_i) \neq y_i) \} \leq \tau. \nonumber
\end{align}

\paragraph{Surrogate Problem} Note that the above objective in \eqref{eq:zero_one_loss_opt} is intractable due to the $0/1$ loss and here we propose a tractable relaxation, replacing the $0/1$ loss with a surrogate loss as follows: 
\begin{align}
      \inf_\theta & \frac{1}{m} \sum_{i=1}^{m} \L_\text{ood} (g_\theta(\tilde{\*x}_i), \text{in}) \label{eq:classification_approach} \\ 
    \text{s.t.} & \frac{1}{n} \sum_{j=1}^n \L_\text{ood} (g_\theta({\*x}_j), \text{out}) \leq \alpha \nonumber \\
    &  \frac{1}{n} \sum_{j=1}^n \L_\text{cls}(f_\theta(\*x_j), y_j) \leq \tau \nonumber
\end{align}
where $\L_\text{ood}$ denotes the loss of the binary OOD classifier and $\L_\text{cls}$ denotes a loss for the classification task. Our framework is general and can be instantiated by different forms of loss functions, for which we describe details later in Section~\ref{sec:neuralnets}.

Here, we state a theoretical result justifying the optimization problem \eqref{eq:classification_approach} where we use tractable losses, specifically using the sigmoid loss $\sigma(t) = \frac{1}{1+e^{-t}}$ for $\L_\text{ood}$ and the hinge loss for $\L_\text{cls}$.
We suppose that the weights are $p$-dimensional and belong to a subset $\theta \in \Theta \subset \R^p$. Let \texttt{opt} denote the value to the population-level optimization problem of interest:
  \begin{align}
        \inf_{\theta \in \Theta} & ~ \mathbb{E}_{\*x \sim \P_\text{out}}\L_\text{ood} (g_\theta(\tilde{\*x}_i), \text{in}) \label{eq:goal_optimization_problem_tractable} \\
        \text{s.t. } &  \mathbb{E}_{\*x \sim \Pin}\L_\text{ood} (g_\theta({\*x}_j), \text{out}) \leq \alpha \nonumber \\
        & \mathbb{E}_{(\*x, y) \sim \P_{\mathcal{X}\mathcal{Y}}}\L_\text{cls}(f_\theta(\*x_j), y_j) \leq \tau. \nonumber
    \end{align}

\vspace{0.3cm}
\begin{proposition}\label{prop:sigmoid_loss}
Suppose $K=2$. Suppose $f_\theta(x) : \R^d \mapsto [0,1]$. Suppose $\L_\text{ood} (t, \text{in}) = \sigma(-t)$, $\L_\text{ood} (t, \text{out}) = \sigma(t)$, and $\L_\text{cls}(s, y)= y \log(\frac{1}{s}) + (1-y) \log(\frac{1}{1-s})$ for $s \in [0,1]$. 
Define $\epsilon_k :=\sqrt{\frac{2 \ln(6/\delta)}{k}} +2\underset{h \in \{f,g\}}{\max} \underset{\P \in \{\P_1, \P_\text{wild}\}}{\max} \underset{{\tilde{\*x}_{1},\ldots, \tilde{\*x}_{m} \sim \P}}{\E} \,\underset{{\eta_1, \ldots, \eta_m}}{\E} \underset{\theta \in \Theta}{\sup} \frac{1}{k} \sum_{i=1}^k \eta_i h_\theta(\tilde{\*x}_{i}) $ 
where $\eta_1,\ldots, \eta_k$ are \emph{i.i.d.} and $\P(\eta_i=1) = \P(\eta_i=-1)=1/2$. Let $\widehat{\theta}_\epsilon$ solve \eqref{eq:classification_approach} with some tolerance $\epsilon $  (see the Appendix for a formal statement). Then, there exist universal positive constants $c_1,c_2,c_3$ such that under a mild condition (see the appendix), with probability at least $1-\delta$
\begin{enumerate}
\item $\E_{\*x \sim \P_\text{out}} \sigma(-\widehat{g}_{\widehat{\theta}_\epsilon}(\*x)) \leq \texttt{opt} + c_1 \pi^{-1}(\epsilon_n + \epsilon_m)$,
    \item $\E_{\*x\sim \P_\text{in}}\sigma(g_{\widehat{\theta}_\epsilon}(\*x))  \leq \alpha + c_2 \epsilon_n$, and
    \item $\E_{(\*x,y)\sim \P_{\mathcal{X}\mathcal{Y}}} \L_\text{cls}(f_{\widehat{\theta}_\epsilon}(\*x), y)  \leq \tau + c_3 \epsilon_n $.
\end{enumerate}
\end{proposition} 
The above Proposition shows that as long as the Rademacher complexities of the function classes $\{g_\theta : \theta \in \Theta\}$ and $\{f_\theta : \theta \in \Theta\}$ decay at a suitable rate, then solving \eqref{eq:classification_approach} gives a solution that approaches feasibility and optimality for \eqref{eq:goal_optimization_problem_tractable}, the optimization problem of interest. Due to space constraints, we defer a full Proposition statement with additional details and proof to the Appendix. The above Proposition extends Theorem 1 of \cite{blanchard2010semi} from the computationally intractable $0/1$ loss to the sigmoid loss, a tractable, differentiable loss that we use in our experiments. In addition, we replace the VC dimension in their Theorem with the Rademacher complexity of the function class, a much more fine-grained measure of the function class complexity. We note that it is also possible to prove an analogue of Proposition \ref{prop:sigmoid_loss} for the $0/1$ loss.

{Although deep neural networks tend to have a trivial Rademacher complexity scaling with the number of data points due to their ability to interpolate random labels \cite{zhang2021understanding}, our training approach has several regularizing features that limit the expressivity of the deep neural network. First, in our training procedure described in Section \ref{experiments}, we add a weight decay term that penalizes deep neural networks with very large weights and limits the network's expressiveness. Second, we perform the classification task and the OOD detection task at once and we conjecture that requiring our algorithm to have good classification performance regularizes the model for the OOD task. Finally, we also fine-tune from a network pre-trained on the classification task with a relatively small learning rate, limiting the optimization procedure's ability to arrive at solutions interpolating random labels and therefore adding further regularization.}

Having theoretically justified the optimization problem \eqref{eq:classification_approach}, we next show how to optimize it.

\subsection{Solving the Constrained Optimization}
\label{sec:optimization}
In this subsection, we first provide background on the Augmented Lagrangian method, and then discuss how it can solve our constrained optimization problem.

\paragraph{Augmented Lagrangian Method (ALM)} Augmented Lagrangian method~\cite{hestenes1969multiplier} is an established approach to optimization with functional constraints. ALM improves over two other related methods for constrained optimization: the penalty method and the method of Lagrangian multipliers. While the penalty method suffers from ill-conditioning \cite{nocedal2006numerical}, the method of Lagrangian multipliers is specific to the convex case \cite{rockafellar1973dual}.
In this paper, we adapt ALM to our setting with inequality constraints, and later show that it can be optimized end-to-end with modern neural networks. 

To provide some background, we consider the following constrained optimization problem as an example:
    \begin{align}
        \min_{\theta \in \R^p} & f(\theta) \label{eq:alm_desc_prob}   \\
        & \text{s.t. } c_i(\theta) \leq 0 \, \forall i \in [q] \nonumber,
    \end{align}
    where $f$ and $c_1,\ldots, c_q$ are convex. ALM solves the constrained optimization problem in \eqref{eq:alm_desc_prob} by converting it into a sequence of unconstrained optimization problems.

Define the the classical augmented Lagrangian (AL) function
\begin{align*}
    \L_\beta(\theta,\lambda) & = f(\theta) + \sum_{i=1}^q \psi_\beta(c_i(\theta),\lambda_i)
\end{align*}
where
\begin{align*}
    \psi_\beta(u,v) & = \begin{cases}
    uv + \frac{\beta}{2} u^2 & \beta u + v \geq 0 \\
    - \frac{v^2}{2 \beta} & \text{o/w}
    \end{cases},
\end{align*}
$\lambda = (\lambda_1,\ldots, \lambda_q)^\t$, and $\beta > 0$. At iteration $k$, ALM minimizes the augmented Lagrangian function with respect to $\theta$ and then performs a gradient ascent update step on $\lambda$ \cite{xu2021first,xu2021iteration}:
\begin{enumerate}
    \item $\theta^{(k+1) } \longleftarrow \argmin_\theta \L_{\beta_k}(\theta,\lambda^{(k)})$
    \item $\lambda^{(k+1)} \longleftarrow \lambda^{(k)} + \rho \nabla_{\lambda} \L_{\beta_k}(\theta^{(k+1)},\lambda),$
\end{enumerate}
where $\rho$ is a learning rate for the dual variable $\lambda$ and $\{\beta_k\}_k$ is a sequence of penalty parameters. The sequence of penalty parameters $\{\beta_k\}_k$ may be chosen beforehand or adapted based on the optimization process.

\paragraph{Our Algorithm}

\begin{algorithm}[t]
\caption{\texttt{WOODS} (Wild OOD detection sans-Supervision)}\label{alg:alm}
	\begin{algorithmic}[1]
		\State \textbf{Input:} $\theta_{(1)}^{(1)}$, $\lambda_{(1)}^{(1)}$ $\beta_1$, $\beta_2$, epoch length $T$, batch size $B$, learning rate $\mu_1$, learning rate $\mu_2$, penalty multiplier $\gamma$, \texttt{tol}
		\For{\text{epoch} $=1,2,\ldots$}
		\For{$t=1,2,\ldots,T-1$}
		\State Sample a batch of data, calculate $\L_\beta^\text{batch}(\theta, \lambda)$
		\State 
		    $\theta^{(t+1)}_{(\text{epoch})} \longleftarrow  \theta^{(t)}_{(\text{epoch})}  - \mu_1 \nabla_\theta  \L_\beta^\text{batch}(\theta, \lambda)$
		\EndFor 
		\State $\lambda^{(\text{epoch}+1)} \leftarrow\lambda^{(\text{epoch})} + \mu_2 \nabla_\theta \L_{\beta}(\theta_{(\text{epoch})}^{(T)}, \lambda^{(\text{epoch})})$ \label{line:grad_asc}
		\If{$\frac{1}{n} \sum_{i=1}^n  \L_\text{ood}( g_{\theta^{(T)}_{(\text{epoch})}}(\*x_i), \text{out})  > \alpha+\texttt{tol}$}
		\State $\beta_1 \longleftarrow \gamma \beta_1$ \label{line:in_constraint_weight_inc}
		\EndIf
		\If{$\frac{1}{n} \sum_{i=1}^n \L_\text{cls}( f_{\theta^{(T)}_{(\text{epoch})}}(\*x_i) , y_i) > \tau+\texttt{tol}$}
		\State $\beta_2 \longleftarrow \gamma \beta_2$ \label{line:class_constraint_weight_inc}
		\EndIf
		\State $\theta^{(1)}_{(\text{epoch}+1)} \longleftarrow \theta^{(T)}_{(\text{epoch})}$
		\EndFor
	\end{algorithmic}
\end{algorithm}

Algorithm \ref{alg:alm} presents our approach to using ALM to solve  \eqref{eq:classification_approach}. We define the augmented Lagrangian function as: 
\begin{align*}
    \L_\beta(\theta,\lambda) & = \frac{1}{m} \sum_{i =1}^m \L_\text{ood} (g_\theta(\tilde{\*x}_i), \text{in}) & \\
    & + \psi_{\beta_1}(\frac{1}{n} \sum_{j =1}^n\L_\text{ood} (g_\theta({\*x}_j), \text{out}) - \alpha, \lambda_{1}) \\
    & + \psi_{\beta_2}(\frac{1}{n} \sum_{j =1}^n \L_\text{cls}(f_\theta(\*x_j), y_j) - \tau, \lambda_{2}),
\end{align*}
where $\beta = (\beta_1, \beta_2)^\t$. 

\paragraph{Adaptation to Stochastic Gradient Descent} We show that our framework can be adapted to the stochastic case, which is more amenable for training with modern neural networks. We outline the full process in Algorithm~\ref{alg:alm}. In each iteration, we calculate the per-batch loss as follows:
\begin{align}
\L_\beta^\text{batch}(\theta, \lambda)& =  \frac{1}{B} \sum_{i \in I} \L_\text{ood} (g_\theta(\tilde{\*x}_i), \text{in}) \label{eq:alm_approach_obj}\\
  & + \psi_{\beta_1}(\frac{1}{B} \sum_{j \in J}\L_\text{ood} (g_\theta({\*x}_j), \text{out}) - \alpha, \lambda_{1}^{(\text{epoch})}) \nonumber \\
  & + \psi_{\beta_2}(\frac{1}{B} \sum_{j \in J}  \L_\text{cls}(f_\theta(\*x_j), y_j) - \tau, \lambda_{2}^{(\text{epoch})})], \nonumber
\end{align}
where $I$ and $J$ denote the set of mini-batch of size $B$, randomly sampled from the wild data and ID data respectively. Since $\psi(u,v)$ is convex in $u$, by Jensen's inequality, the objective function in \eqref{eq:alm_approach_obj} is an upper bound on $\L_\beta(\theta,\lambda^{(\text{epoch})})$. This step therefore approximates ALM; indeed, it is not straightforward to adapt ALM to the stochastic case~\cite{yan2020adaptive}.  At the end of the epoch, it performs a gradient ascent update on $\lambda$ (see line \ref{line:grad_asc}). Finally, in lines \ref{line:in_constraint_weight_inc} and \ref{line:class_constraint_weight_inc}, it increases the constraints weight penalties $\beta_1$ and $\beta_2$ by a penalty multiplier $\gamma > 1$.

\section{Loss Functions with Neural Networks}
\label{sec:neuralnets}

In this section, we discuss how to realize our learning framework in the context of modern neural networks. Concretely, we address how to define the loss functions  $\L_\text{cls}$ and $\L_\text{ood}$.

\paragraph{Classification Loss $\L_\text{cls}$}  
We consider a neural network parameterized by $\theta$, which encodes an input $\*x \in \mathbb{R}^d$ to a feature space with dimension $r$. We denote by $ h_\theta(\*x) \in \mathbb{R}^r$ the feature vector from the penultimate layer of the network. A weight matrix $\mathbf{W} \in \mathbb{R}^{r\times K}$ connects the feature $h_\theta(\*x)$ to the output $f_\theta(\*x) = \mathbf{W}^\top  h_\theta(\*x) \in \mathbb{R}^K$. The per-sample classification loss $\mathcal{L}_\text{cls}$ can be defined using the cross-entropy (CE) loss:
\begin{align}
    \mathcal{L}_\text{cls} (f_\theta(\*x),y) &= - \log{ \frac{e^{f_\theta^{(y)}(\*x)}}{\sum_{j=1}^K e^{f_\theta^{(j)}(\*x)}}},
    \label{eq:ce}
\end{align}
where $f_\theta^{(y)}(\*x)$ denotes the $y$-th element of $f_\theta(\*x)$ corresponding to the label $y$.

%%%%%%%%%%%%%%%%%%%%%%%  TABLE CIFAR %%%%%%%%%%%%%%%%%%%%%%%

\begin{table*}[t]
\centering
\scalebox{0.64}{
\begin{tabular}{llllllllllllllll} \toprule
\multicolumn{1}{c}{\multirow{3}{*}{\textbf{Method}}} & \multicolumn{10}{c}{\textbf{OOD Dataset}} & \multicolumn{2}{c}{\multirow{2}{*}{\textbf{Average}}} & \multirow{3}{*}{\textbf{Acc.}} \\
\multicolumn{1}{c}{} &
\multicolumn{2}{c}{\textbf{SVHN}} & 
\multicolumn{2}{c}{\textbf{LSUN-R}} & 
\multicolumn{2}{c}{\textbf{LSUN-C}} &
\multicolumn{2}{c}{\textbf{Textures}} & \multicolumn{2}{c}{\textbf{Places365}} & \multicolumn{2}{c}{} & \\
\multicolumn{1}{c}{} & {FPR$\downarrow$} & {AUROC$\uparrow$} & {FPR$\downarrow$} & {AUROC$\uparrow$} & {FPR$\downarrow$} & {AUROC$\uparrow$} & {FPR$\downarrow$} & {AUROC$\uparrow$} & {FPR$\downarrow$} & {AUROC$\uparrow$} & {FPR$\downarrow$} & {AUROC$\uparrow$} & \\ \midrule &\multicolumn{13}{c}{\textbf{With $\Pin$ only}}          \\
MSP & 84.59 & 71.44 & 82.42 & 75.38 & 66.54 & 83.79 & 83.29 & 73.34 & 82.84 & 73.78 & 79.94 & 75.55 & \textbf{75.96} \\
ODIN & 84.66 & 67.26 & 71.96 & 81.82 & 55.55 & 87.73 & 79.27 & 73.45 & 87.88 & 71.63 & 75.86 & 76.38 & \textbf{75.96}\\
Energy & 85.82 & 73.99 & 79.47 & 79.23 & 35.32 & 93.53 & 79.41 & 76.28 & 80.56 & 75.44 & 72.12 & 79.69 & \textbf{75.96} \\
Mahalanobis & 57.52 & 86.01 & 21.23 & 96.00 & 91.18 & 69.69 & 39.39 & 90.57 & 88.83 & 67.87 & 59.63 & 82.03 & \textbf{75.96}\\ 
GODIN & 83.38 & 84.05 & 62.24 & 88.22 & 72.86 & 83.84 & 83.83 & 78.91  & 80.56 & 76.14 & 76.57 & 82.23 & 75.33\\
CSI & 64.70 & 84.97 & 91.55 & 63.42 & 38.10 & 92.52 & 74.70 & 92.66 & 82.25 & 73.63 & 70.26 & 81.44 & 69.90\\
\midrule &\multicolumn{13}{c}{\textbf{With $\Pin$ and $\Pwild$}} \\
OE                       &           1.57$^{\pm 0.1}$ &           99.63$^{\pm 0.0}$ &           0.93$^{\pm 0.2}$ &           99.79$^{\pm 0.0}$ &           3.83$^{\pm 0.4}$ &           99.26$^{\pm 0.1}$ &           27.89$^{\pm 0.5}$ &           93.35$^{\pm 0.2}$ &           60.24$^{\pm 0.6}$ &           83.43$^{\pm 0.6}$ &
18.89$^{\pm 0.4}$ & 95.09$^{\pm 0.2}$ &           71.65$^{\pm 0.4}$ \\
Energy (w/ OE)           &           1.47$^{\pm 0.3}$ &           99.68$^{\pm 0.0}$ &           2.68$^{\pm 1.9}$ &           99.50$^{\pm 0.3}$ &           2.52$^{\pm 0.4}$ &           99.44$^{\pm 0.1}$ &          37.26$^{\pm 9.1}$ &           91.26$^{\pm 2.5}$ &           54.67$^{\pm 1.0}$ &           86.09$^{\pm 0.4}$ & 19.72$^{\pm 2.5}$ & 95.19$^{\pm 0.7}$ &           73.46$^{\pm 0.8}$ \\
WOODS (ours)             &           0.52$^{\pm 0.1}$ &           99.88$^{\pm 0.0}$ &           0.38$^{\pm 0.1}$ &           99.92$^{\pm 0.0}$ &           0.93$^{\pm 0.2}$ &           99.77$^{\pm 0.0}$ &          17.92$^{\pm 0.5}$ &           96.44$^{\pm 0.2}$ &           37.90$^{\pm 0.6}$ &  \textbf{91.22}$^{\pm 0.3}$ & 11.53$^{\pm 0.3}$ & 97.45$^{\pm 0.1}$ &           74.79$^{\pm 0.2}$ \\
WOODS-alt (ours) &  \textbf{0.12}$^{\pm 0.0}$ &  \textbf{99.96}$^{\pm 0.0}$ &  \textbf{0.07}$^{\pm 0.1}$ &  \textbf{99.96}$^{\pm 0.0}$ &  \textbf{0.11}$^{\pm 0.0}$ &  \textbf{99.96}$^{\pm 0.0}$ &  \textbf{9.12}$^{\pm 0.3}$ &  \textbf{96.65}$^{\pm 0.1}$ &  \textbf{29.58}$^{\pm 0.4}$ &           90.60$^{\pm 0.3}$ & \textbf{7.80}$^{\pm 0.5}$  & \textbf{97.43}$^{\pm 0.5}$ &  \textbf{75.22}$^{\pm 0.2}$ \\
 \bottomrule
\end{tabular}
}
\caption{\small \textbf{Main results when $\P_\text{out}^\text{test} = \P_\text{out}$.} Comparison with competitive OOD detection methods on \texttt{CIFAR-100}. For methods using $\Pwild$, we train under the same dataset and same $\pi=0.1$. For each dataset, we create corresponding wild mixture distribution $\P_\text{wild} := (1-\pi) \P_\text{in} + \pi \P_\text{out}$ for training and test on the corresponding OOD dataset. $\uparrow$ indicates larger values are better and vice versa. $\pm x$ denotes the standard error, rounded to the first decimal point.} 
\label{tab:cifar100_main}
\end{table*}

\paragraph{Binary Loss $\L_\text{ood}$} The loss function $\L_\text{ood}$ should ideally optimize for the separability between the ID vs. OOD data under some function that captures the data density. However, directly estimating $\log p(\*x)$ can be computationally intractable as it requires sampling from the entire space $\mathcal{X}$. We note that the log partition function $E_\theta(\*x) := -\log \sum_{j=1}^K e^{f_{\theta}^{(j)}(\*x)}$ is proportional to  $-\log p(\*x)$ with some unknown factor, which can be seen from the following:
$$
p(y | \*x)  = \frac{p(\*x,y)}{p(\*x)} = \frac{e^{f_\theta^{(y)}(\*x)}}{\sum_{j=1}^K e^{f_\theta^{(j)}(\*x)}}.
$$
The negative log partition function is also known as the {free energy}, which was shown to be an effective uncertainty measurement for OOD detection~\cite{liu2020energy}. 
    
Our idea is to explicitly optimize for a level-set estimation based on the energy function (threshold  at 0), where the ID data has negative energy values and vice versa. 
\begin{align*}
            \argmin_{\theta} & \frac{1}{m}\sum_{i=1}^{m} \d1\{ E_\theta(\tilde{\*x}_i) \leq 0\}  \\
         \text{s.t. } & \frac{1}{n}\sum_{j=1}^n \d1\{ E_\theta(\*x_j) \geq 0\} \leq \alpha \\
         &  \frac{1}{n} \sum_{j=1}^n \L_\text{cls}(f_\theta(\*x_j), y_j) \leq \tau
\end{align*}
Since the $0/1$ loss is intractable,  we replace it with the binary sigmoid loss, a smooth approximation of the $0/1$ loss, yielding the following optimization problem:
\begin{align}
 \argmin_{\theta, w \in \R} & \frac{1}{m}\sum_{i=1}^{m} \frac{1}{1 + \exp(-w \cdot E_\theta(\tilde{\*x}_i))} \label{eq:energy_robust} \\
         \text{s.t. } & \frac{1}{n}\sum_{j=1}^n \frac{1}{1+\exp(w \cdot E_\theta(\*x_i))}   \leq \alpha \nonumber \\
        & \frac{1}{n} \sum_{j=1}^n \L_\text{cls}( f_{\theta}(\*x_j) , y_j) \leq \tau. \nonumber
\end{align}
Here $w$ is a learnable parameter modulating the slope of the sigmoid function. Now we may apply the same approach as in section \ref{sec:method} to solve the constrained optimization problem \eqref{eq:energy_robust}. In effect, we have $$\L_\text{ood}(g_\theta(\tilde{\*x}_i), \text{in}) = \frac{1}{1 + \exp(-w \cdot E_\theta(\tilde{\*x}_i))}.$$

This loss function is originally developed in~\cite{du2022vos} for model regularization. Our loss has two notable advancements over energy-regularized learning (ERL)~\cite{liu2020energy}: (1) We consider a more general unsupervised setting where the wild data distribution $\Pwild$ contains both ID and OOD data, whereas ERL assumes having access to an auxiliary outlier dataset that is completely separable from the ID data. Methods like \texttt{OE}~\cite{hendrycks2018deep} in general require performing manual data collection and cleaning, which is more restrictive. (2) We formalize the learning objective as a constrained optimization, which offers strong guarantees (see Section~\ref{sec:objective}).

\section{Experiments}
\label{experiments}

\subsection{Experimental Setup}
\paragraph{Datasets} 
Following the common benchmarks in literature, we use \texttt{CIFAR-10} and \texttt{CIFAR-100} \cite{krizhevsky_learning_2009} as ID datasets ($\Pin$). For OOD test datasets $\P_\text{out}^\text{test}$, we use a suite of natural image datasets including \texttt{SVHN} \cite{netzer_reading_2011}, \texttt{Textures} \cite{cimpoi_describing_2014}, \texttt{Places365} \cite{zhou_places_2018}, \texttt{LSUN-Crop} \cite{yu_lsun_2016}, and \texttt{LSUN-Resize} \cite{yu_lsun_2016}.

To simulate the wild data $\Pwild$, we mix a subset of ID data (as $\P_\text{in}$) with the outlier dataset (as $\P_\text{out}$) under various $\pi \in \{0.05, 0.1, 0.2, 0.5, 1.0\}$. We consider the stationary setting where $\P_\text{out}^\text{test} = \P_\text{out}$, and the non-stationary setting where $\P_\text{out}^\text{test} \neq \P_\text{out}$. When $\P_\text{out}^\text{test} = \P_\text{out}$ (say from \texttt{SVHN}), we train the model using \texttt{CIFAR+SVHN} as the wild data and test on \texttt{SVHN} as OOD. When $\P_\text{out}^\text{test} \neq \P_\text{out}$, for $\P_\text{out}$, we use the publicly available \texttt{300K Random Images}~\cite{hendrycks2018deep}, a subset of the original 80 Million Tiny Images dataset, and test on the 5 OOD datasets mentioned above. Note that we split \texttt{CIFAR} datasets into two halves: 25,000 images as ID training data, and remainder 25,000 used to create the wild mixture data.

\paragraph{Evaluation Metrics} To evaluate the methods, we use the standard measures in the literature: the false positive rate of declaring OOD examples as ID when $95\%$ of ID data points are declared as ID (FPR95) and the area under the receiver operating characteristic curve (AUROC). 

\begin{table*}[t]
\centering
\scalebox{0.64}{
\begin{tabular}{llllllllllllllll} \toprule
\multicolumn{1}{c}{\multirow{3}{*}{\textbf{Method}}} & \multicolumn{12}{c}{\textbf{OOD Dataset}} &  \multirow{3}{*}{\textbf{Acc.}} \\
\multicolumn{1}{c}{} &
\multicolumn{2}{c}{\textbf{SVHN}} & 
\multicolumn{2}{c}{\textbf{LSUN-R}} & 
\multicolumn{2}{c}{\textbf{LSUN-C}} &
\multicolumn{2}{c}{\textbf{Textures}} & \multicolumn{2}{c}{\textbf{Places365}} & \multicolumn{2}{c}{\textbf{300K Rand. Img.}} & \\
\multicolumn{1}{c}{} & {FPR$\downarrow$} & {AUROC$\uparrow$} & {FPR$\downarrow$} & {AUROC$\uparrow$} & {FPR$\downarrow$} & {AUROC$\uparrow$} & {FPR$\downarrow$} & {AUROC$\uparrow$} & {FPR$\downarrow$} & {AUROC$\uparrow$} & {FPR$\downarrow$} & {AUROC$\uparrow$} & \\ \midrule &\multicolumn{13}{c}{$\pi=0.05$}          \\
OE             &           80.21$^{\pm 1.7}$ &           77.47$^{\pm 1.8}$ &           77.97$^{\pm 2.3}$ &           78.68$^{\pm 1.7}$ &           61.27$^{\pm 1.4}$ &           86.27$^{\pm 0.4}$ &           77.15$^{\pm 1.2}$ &           77.94$^{\pm 0.5}$ &           80.24$^{\pm 0.3}$ &           74.86$^{\pm 0.2}$ &            75.33$^{\pm 0.3}$ &           77.16$^{\pm 0.3}$ &           73.63$^{\pm 0.3}$ \\
Energy (w/ OE) &           77.47$^{\pm 2.0}$ &           80.48$^{\pm 1.2}$ &           70.83$^{\pm 3.1}$ &           82.86$^{\pm 2.0}$ &           29.42$^{\pm 4.3}$ &           94.61$^{\pm 0.8}$ &           72.05$^{\pm 0.8}$ &           80.73$^{\pm 0.5}$ &           74.69$^{\pm 0.6}$ &           78.60$^{\pm 0.4}$ &            66.91$^{\pm 0.7}$ &           80.44$^{\pm 0.5}$ &           75.77$^{\pm 0.1}$ \\
WOODS (ours)   &  \textbf{74.54}$^{\pm 1.7}$ &  \textbf{82.01}$^{\pm 1.3}$ &  \textbf{66.29}$^{\pm 3.9}$ &  \textbf{84.46}$^{\pm 2.3}$ &  \textbf{19.07}$^{\pm 1.6}$ &  \textbf{96.48}$^{\pm 0.3}$ &  \textbf{65.75}$^{\pm 0.6}$ &  \textbf{83.71}$^{\pm 0.2}$ &  \textbf{69.97}$^{\pm 1.1}$ &  \textbf{80.82}$^{\pm 0.5}$ &   \textbf{62.48}$^{\pm 1.1}$ &  \textbf{82.92}$^{\pm 0.5}$ &  \textbf{75.92}$^{\pm 0.1}$ \\
\midrule &\multicolumn{13}{c}{$\pi=0.1$}          \\
OE             &           79.56$^{\pm 1.6}$ &           77.00$^{\pm 1.2}$ &           76.86$^{\pm 2.1}$ &           78.75$^{\pm 1.2}$ &           58.53$^{\pm 2.8}$ &           86.92$^{\pm 0.8}$ &           74.63$^{\pm 1.2}$ &           79.13$^{\pm 0.5}$ &           78.52$^{\pm 0.3}$ &           75.68$^{\pm 0.1}$ &            72.18$^{\pm 0.2}$ &           78.48$^{\pm 0.3}$ &           73.53$^{\pm 0.4}$ \\
Energy (w/ OE) &           77.45$^{\pm 2.1}$ &           80.94$^{\pm 1.4}$ &           67.13$^{\pm 3.6}$ &           83.68$^{\pm 2.4}$ &           27.08$^{\pm 2.1}$ &           94.97$^{\pm 0.4}$ &           70.15$^{\pm 1.0}$ &           81.59$^{\pm 0.6}$ &           71.71$^{\pm 1.1}$ &           79.89$^{\pm 0.6}$ &            64.24$^{\pm 2.3}$ &           82.28$^{\pm 1.1}$ &           75.27$^{\pm 0.2}$ \\
WOODS (ours)   &  \textbf{71.67}$^{\pm 1.9}$ &  \textbf{84.11}$^{\pm 1.4}$ &  \textbf{59.27}$^{\pm 3.9}$ &  \textbf{86.80}$^{\pm 1.9}$ &  \textbf{15.03}$^{\pm 1.4}$ &  \textbf{97.24}$^{\pm 0.3}$ &  \textbf{61.38}$^{\pm 0.7}$ &  \textbf{85.57}$^{\pm 0.2}$ &  \textbf{64.19}$^{\pm 1.0}$ &  \textbf{83.12}$^{\pm 0.5}$ &   \textbf{55.51}$^{\pm 1.3}$ &  \textbf{85.72}$^{\pm 0.4}$ &  \textbf{75.64}$^{\pm 0.3}$ \\
\midrule &\multicolumn{13}{c}{$\pi=0.2$}          \\
OE             &           72.59$^{\pm 3.9}$ &           81.38$^{\pm 1.9}$ &           65.04$^{\pm 3.8}$ &           82.65$^{\pm 1.8}$ &           48.62$^{\pm 3.1}$ &           89.52$^{\pm 0.8}$ &           65.95$^{\pm 1.2}$ &           82.43$^{\pm 0.3}$ &           71.29$^{\pm 0.7}$ &           78.71$^{\pm 0.4}$ &            65.40$^{\pm 0.8}$ &           81.99$^{\pm 0.1}$ &           72.89$^{\pm 0.3}$ \\
Energy (w/ OE) &           72.76$^{\pm 2.5}$ &           83.48$^{\pm 1.2}$ &           62.53$^{\pm 5.7}$ &           84.46$^{\pm 2.8}$ &           22.49$^{\pm 1.2}$ &           95.84$^{\pm 0.2}$ &           64.93$^{\pm 0.5}$ &           83.87$^{\pm 0.4}$ &           64.62$^{\pm 0.2}$ &           82.72$^{\pm 0.2}$ &            56.07$^{\pm 1.2}$ &           85.50$^{\pm 0.4}$ &           75.00$^{\pm 0.3}$ \\
WOODS (ours)   &  \textbf{71.61}$^{\pm 2.3}$ &  \textbf{84.99}$^{\pm 1.2}$ &  \textbf{51.66}$^{\pm 2.8}$ &  \textbf{89.68}$^{\pm 1.2}$ &  \textbf{12.63}$^{\pm 0.6}$ &  \textbf{97.67}$^{\pm 0.1}$ &  \textbf{59.77}$^{\pm 0.5}$ &  \textbf{86.74}$^{\pm 0.1}$ &  \textbf{58.29}$^{\pm 0.4}$ &  \textbf{85.22}$^{\pm 0.1}$ &   \textbf{49.87}$^{\pm 1.8}$ &  \textbf{88.25}$^{\pm 0.2}$ & \textbf{75.26}$^{\pm 0.2}$ \\
\midrule &\multicolumn{13}{c}{$\pi=0.5$}          \\
OE             &  \textbf{68.80}$^{\pm 2.8}$ &           82.89$^{\pm 1.1}$ &           47.64$^{\pm 4.7}$ &           88.84$^{\pm 1.8}$ &           30.86$^{\pm 1.9}$ &           93.91$^{\pm 0.4}$ &  \textbf{56.18}$^{\pm 1.6}$ &           86.11$^{\pm 0.4}$ &           62.24$^{\pm 0.5}$ &           82.53$^{\pm 0.3}$ &            53.70$^{\pm 1.6}$ &           86.58$^{\pm 0.2}$ &            73.00$^{\pm 0.3}$ \\
Energy (w/ OE) &           69.81$^{\pm 2.4}$ &           85.59$^{\pm 1.0}$ &           56.11$^{\pm 3.1}$ &           87.41$^{\pm 1.5}$ &           16.23$^{\pm 0.6}$ &           97.02$^{\pm 0.1}$ &           58.41$^{\pm 0.9}$ &           86.70$^{\pm 0.1}$ &           58.31$^{\pm 0.5}$ &           85.36$^{\pm 0.4}$ &            48.12$^{\pm 1.3}$ &           88.76$^{\pm 0.3}$ &           74.87$^{\pm 0.4}$ \\
WOODS (ours)   &           69.41$^{\pm 2.7}$ &  \textbf{86.76}$^{\pm 0.8}$ &  \textbf{44.60}$^{\pm 2.6}$ &  \textbf{91.72}$^{\pm 0.7}$ &  \textbf{12.70}$^{\pm 0.4}$ &  \textbf{97.71}$^{\pm 0.1}$ &           57.60$^{\pm 0.6}$ &  \textbf{87.74}$^{\pm 0.1}$ &  \textbf{55.03}$^{\pm 0.3}$ &  \textbf{86.82}$^{\pm 0.1}$ &   \textbf{45.00}$^{\pm 0.7}$ &  \textbf{89.85}$^{\pm 0.2}$ &  \textbf{75.72}$^{\pm 0.0}$ \\
\midrule &\multicolumn{13}{c}{$\pi=1.0$}          \\
OE             &  \textbf{46.45}$^{\pm 2.7}$ &  \textbf{91.82}$^{\pm 0.5}$ &           51.26$^{\pm 3.6}$ &           88.47$^{\pm 1.2}$ &           20.08$^{\pm 0.7}$ &           96.42$^{\pm 0.1}$ &            \textbf{51.31}$^{\pm 0.8}$ &           88.81$^{\pm 0.2}$ &           55.66$^{\pm 0.4}$ &           87.28$^{\pm 0.1}$ &   44.29$^{\pm 0.6}$ &           90.44$^{\pm 0.1}$ &            74.99$^{\pm 0.1}$ \\
Energy (w/ OE) &           56.40$^{\pm 4.0}$ &           89.48$^{\pm 1.2}$ &           54.41$^{\pm 2.5}$ &           88.77$^{\pm 0.8}$ &           17.14$^{\pm 0.9}$ &           96.91$^{\pm 0.1}$ &           52.36$^{\pm 1.3}$ &  \textbf{89.38}$^{\pm 0.3}$ &  \textbf{54.11}$^{\pm 0.9}$ &  \textbf{88.35}$^{\pm 0.2}$ &            \textbf{43.42}$^{\pm 1.0}$ &  \textbf{90.88}$^{\pm 0.1}$ &            74.85$^{\pm 0.2}$ \\
WOODS (ours)   &           62.13$^{\pm 4.4}$ &           88.89$^{\pm 1.4}$ &  \textbf{45.87}$^{\pm 1.1}$ &  \textbf{91.64}$^{\pm 0.2}$ &  \textbf{13.48}$^{\pm 1.1}$ &  \textbf{97.58}$^{\pm 0.2}$ &  56.83$^{\pm 0.7}$ &           88.19$^{\pm 0.3}$ &           54.57$^{\pm 0.3}$ &           87.43$^{\pm 0.3}$ &            45.61$^{\pm 3.0}$ &           89.78$^{\pm 1.0}$ &  \textbf{75.60}$^{\pm 0.2}$ \\
\bottomrule
\end{tabular}
}
\caption{\small \textbf{Effect of $\pi$.} ID dataset is \texttt{CIFAR-100}, and the auxiliary outlier training data is \texttt{300K Random Images}.  $\uparrow$ indicates larger values are better and vice versa. $\pm x$ denotes the standard error, rounded to the first decimal point.} 
\label{tab:ood_cifar100}
\end{table*}

\begin{figure*}
    \centering
    \includegraphics[width=0.33\textwidth]{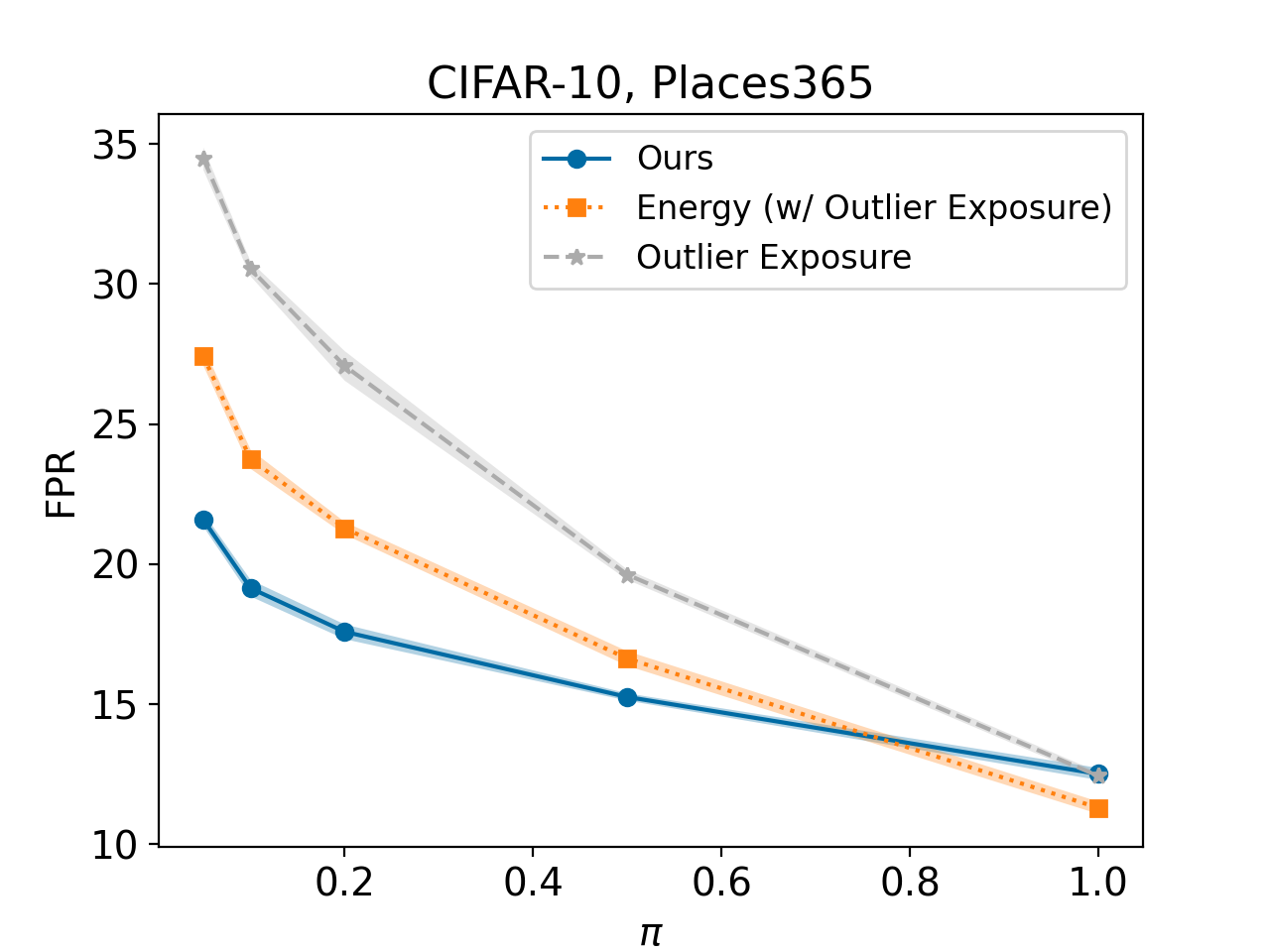}
    \includegraphics[width=0.33\textwidth]{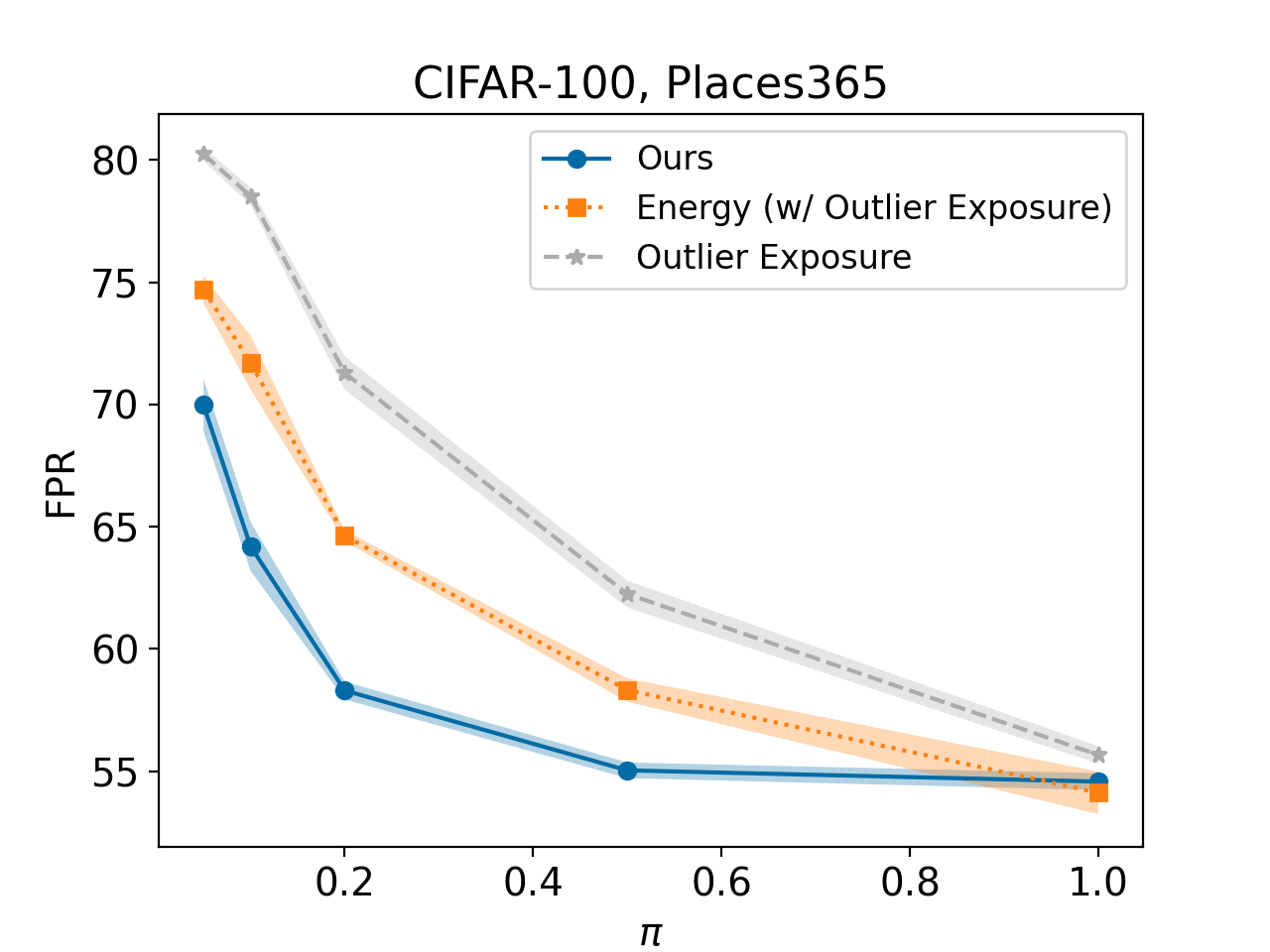}
       \includegraphics[width=0.33\textwidth]{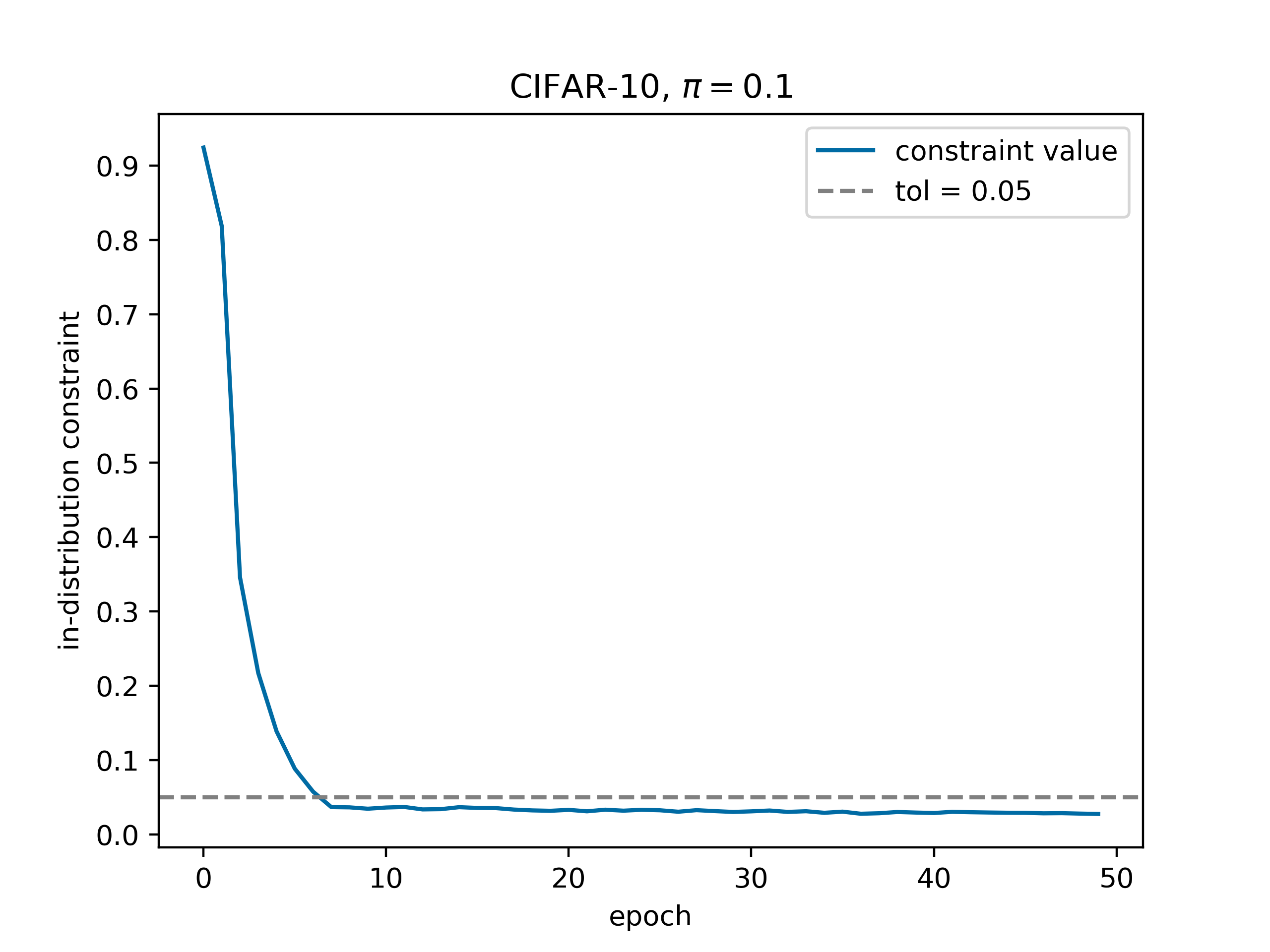}
    \caption{ \textbf{Left and middle}: Ablation on $\pi$ for the OOD setting, using \texttt{CIFAR-10} (left) and \texttt{CIFAR-100} (middle) as the ID dataset and \texttt{Places365} as the OOD dataset. Our method \texttt{WOODS} is more reliable as $\pi$ decreases.  \textbf{Right}: Value of the ID constraint term $\frac{1}{m} \sum_{j=1}^m \frac{1}{1 + \textrm{exp} (w \cdot E_\theta (\*x_i))}  - \alpha$ from \eqref{eq:energy_robust} over different training epochs. Our method is effective in satisfying this constraint, reducing it to zero (within a tolerance of 0.05).}
    \label{fig:pi}
\end{figure*}

\paragraph{Training Details} For all experiments and methods, we use the Wide ResNet~\cite{zagoruyko_wide_2016} architecture with 40 layers and widen factor of 2. The model is optimized using stochastic gradient descent with Nesterov momentum \cite{duchi2011adaptive}. We set the weight decay coefficient to be 0.0005, and momentum to be 0.09. Models are initialized with a model pre-trained on the \texttt{CIFAR-10}/\texttt{CIFAR-100} data and trained for 100 epochs in the $\P_\text{out}^\text{test} = \P_\text{out}$ setting and for 50 epochs in the $\P_\text{out}^\text{test} \neq \P_\text{out}$ setting. Our initialization scheme from a pre-trained model naturally suits our setting (\emph{c.f.} Section~\ref{sec:problem_setup}), where an existing classifier in deployment is available. 

The initial learning rate is set to be 0.001 and decayed by a factor of 2 after 50\%, 75\%, and 90\% of the epochs. We use a batch size of 128 and a dropout rate of $0.3$. All training is performed in PyTorch using NVIDIA GeForce RTX 2080 Ti GPUs. For optimization in \texttt{WOODS}, we vary the penalty multiplier $\gamma \in \{1.1, 1.5\}$ and the dual update learning rate $\mu_2 \in \{0.1,1,2\}$. We set $\texttt{tol}=0.05$, $\alpha=0.05$, and set $\tau$ to be twice the loss of the pre-trained model. For validation, we use subsets of the ID data and of the $\P_\text{out}$ data. Further details are included in Appendix~\ref{sec:exper_details} and code is made publicly available online. 

\subsection{Results}
\label{exp_main}

\paragraph{\texttt{WOODS} Achieves Superior Performance} We present results in Table~\ref{tab:cifar100_main}, where \texttt{WOODS} outperforms the state-of-the-art results in the stationary wild setting $\P_\text{out}^\text{test} = \P_\text{out}$. This setting reflects the practical scenario when the OOD data remains stable. Our comparison covers an extensive collection of competitive OOD detection methods. For clarity, we divide the baseline methods into two categories: trained with and without in-the-wild data. For methods using ID data $\Pin$ only, we compare with methods such as \texttt{MSP}~\cite{Kevin}, \texttt{ODIN}~\cite{liang2018enhancing}, \texttt{Mahalanobis}~\cite{lee2018simple}, and \texttt{Energy}~\cite{liu2020energy}, the model is trained with softmax CE loss, same as in Equation~\ref{eq:ce}.  \texttt{GODIN}~\cite{hsu2020generalized} is trained using a DeConf-C loss, which does not involve auxiliary data loss either. We also include the latest development based on self-supervised losses, namely
\texttt{CSI}~\cite{tack2020csi}.\footnote{We do not include standard errors for the methods that only use the ID data. We note the large difference in performance is statistically significant. We also include the pre-trained model in the Github repository.}

Closest to ours are Outlier Exposure (\texttt{OE})~\cite{hendrycks2018deep} and energy-based OOD learning method~\cite{liu2020energy}. These are among the strongest OOD detection baselines, which regularize the classification model by producing lower confidence or higher energy on the auxiliary outlier data. For a fair comparison, all the methods in this group are trained using the same ID and in-the-wild data, under the same mixture ratio $\pi=0.1$. 

We highlight several observations: (1) Methods using wild data $\Pwild$, in general, show strong OOD detection performance over the counterpart (without $\Pwild$). Compared to the strongest baseline \texttt{Mahalanobis} in the first group, our method outperforms by \textbf{48.10\%} in FPR95, averaged across all test datasets. The performance gain precisely demonstrates the advantage of our setting, which incorporates in-the-wild data for effective OOD learning. (2) Compared to methods using $\Pwild$, \texttt{WOODS} outperforms \texttt{OE} by \textbf{7.36\%} in FPR95, with particularly strong performance on the \texttt{Textures} and \texttt{Places365} datasets. In particular, \texttt{OE} makes a strong distributional assumption that the auxiliary outlier data does not overlap with the ID data. This assumption is violated when using wild mixture data. Indeed, our framework does not make this assumption, and demonstrates superior performance.
(3) Lastly, the ID accuracy of the model trained with our method is comparable to that using the CE loss alone. Due to space constraints, we provide results on \texttt{CIFAR-10} in the Appendix~\ref{sec:appendix_CIFAR10}, where our method's strong performance holds.

\paragraph{Effect of $\pi$} In Table~\ref{tab:pi} and Figure~\ref{fig:pi}, we ablate the effect under different $\pi$, which modulates the fraction of OOD data in the mixture distribution $\Pwild$. Here we consistently use the \texttt{300k Random Images} as the auxiliary outlier training dataset. Recall our definition in Section~\ref{sec:problem_setup}, a smaller $\pi$ indicates more ID data and less OOD data---this reflects the practical scenario that the majority of test data may remain ID. Note that the table covers both settings of  $\P_\text{out}^\text{test} = \P_\text{out}$ (i.e., the column of 300K Rand Img) and $\P_\text{out}^\text{test} \neq \P_\text{out}$ (other columns). We highlight a few interesting observations: (1) The OOD detection performance for all methods (including \texttt{OE} and energy regularized learning) generally degrades as with decreasing $\pi$. In particular, a smaller $\pi$  translates into a harder learning problem, because $\Pin$ and $\Pwild$ become largely overlapping. For example, on the \texttt{300K Random Images} datset, the FPR95 of \texttt{OE} increases from 44.29\% ($\pi=1.0$) to 75.33\% ($\pi=0.05$). (3) Our method \texttt{WOODS} is overall more robust under small $\pi$ settings than the baselines. In a challenging case with $\pi=0.05$, our method outperforms \texttt{OE} by \textbf{12.85\%} in FPR95 on \texttt{300K Random Images}. This demonstrates the benefits of \texttt{WOODS} performing constrained optimization.

\paragraph{\texttt{WOODS} Satisfies the Constraints in Optimization} We also perform a sanity check on whether \texttt{WOODS} satisfies the constraints of the optimization objective in \eqref{eq:energy_robust}. As shown in Figure~\ref{fig:pi} (right), the ID constraint value $\frac{1}{m} \sum_{j=1}^m \frac{1}{1 + \textrm{exp} (w \cdot E_\theta (\*x_i))}  - \alpha$ is reduced to zero, within a specified tolerance of 0.05. This indeed verifies the efficacy of our constrained optimization framework. 

\vspace{-0.2cm}
\paragraph{Alternative Loss} We note that our framework could also be compatible with alternative forms of loss function  $\L_\text{ood}$. For example, we also explored another variant of our method, \texttt{WOODS-alt}, which uses a one hidden layer neural network applied to the penultimate layer for the OOD detection task. Section \ref{sec:results_ssnd_setting} for more details of the loss function. The performance is summarized in Table~\ref{tab:cifar100_main}. While \texttt{WOODS-alt} has strong performance in the setting where $\mathbb{P}_\textrm{out} = \mathbb{P}_\textrm{out}^\textrm{test}$, we do not expect it to perform as well as \texttt{WOODS} in the setting studied previously where $\mathbb{P}_\textrm{out} \neq \mathbb{P}_\textrm{out}^\textrm{test}$. \texttt{WOODS} uses the energy score to build its classifier, which already has reasonable performance even without any additional auxiliary dataset. On the other hand, \texttt{WOODS-alt} would do no better than random guessing without an auxiliary dataset. In this way, \texttt{WOODS} has a prior given by the energy score that we believe helps in the $\mathbb{P}_\textrm{out} \neq \mathbb{P}_\textrm{out}^\textrm{test}$ setting. Indeed, we observed this in some experiments.

\section{Related Work}
\paragraph{OOD Detection} OOD detection is an essential topic for safely deploying machine learning models in the open world, attracting much recent interest. 

1) A rich line of methods design scoring functions for detecting samples from outside training categories, such as OpenMax~\cite{openworld}, Maximum Softmax Probability \cite{Kevin}, ODIN score \cite{liang2018enhancing}, Mahalanobis distance-based score~\cite{lee2018simple}, Energy score \cite{liu2020energy, wang2021can, morteza2022provable, sun2021react}, gradient-based score~\cite{huang2021importance}, and non-parametric KNN score~\cite{sun2022knn}. The aforementioned methods utilize ID data only. We show that by leveraging  wild data that exists naturally in the model's habitats, one can in fact build a stronger OOD detector. 

2) Another line of works address the OOD detection problem by training-time regularization \cite{lee2018gan,  bevandic2018discriminative, hendrycks2018deep, malinin2018predictive, liu2020energy, chen2020informative}. For example, models are encouraged to give predictions with lower confidence~\cite{hendrycks2018deep} or higher energies~\cite{liu2020energy,ming2022posterior}. These methods typically require access to auxiliary outlier dataset. To circumvent this, recent work also explored synthesizing virtual outliers~\cite{du2022vos, du2022unknown}. In this work, we instead explore a more realistic setting by training OOD detectors using wild mixtures data containing both ID and OOD data---which can be easily obtainable upon deploying a machine learning classifier. We formulate a novel constrained optimization problem and show how to solve it tractably with modern neural networks. 

\vspace{-0.4cm}
\paragraph{Anomaly Detection} Anomaly detection has received much attention in recent years (e.g., \cite{ruff2018deep, chalapathy2018anomaly, ergen2019unsupervised, perera2019learning, song2017hybrid}). In anomaly detection, a dataset is drawn \emph{i.i.d.} from $\P_\text{in}$ and the goal is to identify whether new data points are anomalous in the sense that they are not realizations from $\P_\text{in}$. In semi-supervised anomaly detection, an additional clean OOD dataset drawn from $\P_\text{out}$ is observed (e.g., \cite{ruff2019deep, daniel2019deep, hendrycks2018deep}). An important difference between anomaly detection and the OOD detection literature is that OOD detection additionally requires learning a classifier for the distribution $\P_{\mathcal{X}\mathcal{Y}}$. We refer the reader to \cite{ruff2021unifying, chalapathy2019deep} for detailed surveys on anomaly detection.

A closely related paper to our work is \cite{blanchard2010semi}, which studies the setting where samples from $\Pin$ and $\P_\text{wild}$ are observed and the goal is to find a $\theta$ minimizing $ \mathbb{E}_{\*x \sim \P_\text{out}}(\d1{\{g_\theta(\*x)=\text{in}\}})$ subject to the constraint that  $\mathbb{E}_{\*x \sim \Pin}(\d1{\{g_\theta(\*x)=\text{out}\}}) \leq \alpha$. This work has several important differences with ours. First, they do not consider the out-of-distribution problem, that is, where the distribution at test time $\P_\text{out}^\text{test}$ differs from $\P_\text{out}$, whereas our energy-based approach \eqref{eq:energy_robust} does. Second, their formulation only considers the task of distinguishing $\P_\text{out}$ and $\P_\text{in}$, not the task of doing classification simultaneously. Third, our work uses and theoretically analyzes the sigmoid loss for OOD detection, a differentiable loss that can be used in deep learning, whereas their work focuses on the computationally intractable 0-1 loss.  Finally, their work is mainly statistical, only implementing their algorithm using a plug-in kernel-density estimator whereas we leverage neural networks using a computational approach based on ALM.

\vspace{-0.2cm}
\paragraph{Constrained Optimization} The augmented Lagrangian method is a popular approach to constrained optimization. It improves over two other related methods: the penalty method and the method of Lagrangian multipliers. While the penalty method suffers from ill-conditioning \cite{nocedal2006numerical}, the method of Lagrangian multipliers is specific to the convex case \cite{rockafellar1973dual}. In this paper, we adapt ALM to our setting from a recent version proposed and analyzed for the case of nonlinear inequality constraints \cite{xu2021first}. There are only a limited number of examples of adapting ALM to modern neural networks (e.g., \cite{sangalli2021constrained} for class imbalance). To the best of our knowledge, our work is the first to explore constrained optimization for OOD detection problem

\vspace{-0.2cm}
\section{Conclusion}
In this paper, we propose a novel framework for OOD detection using wild data. Wild data has significant promise since \emph{(i)} it is abundant, \emph{(ii)} can be collected essentially for free upon deploying an ML system, and \emph{(iii)} is often a much better match to the test-time distribution than data collected offline. At the same time, it is challenging to leverage because it naturally consists of both ID and OOD examples. To overcome this challenge, we propose a framework based on constrained optimization and solve it tractably by adapting the augmented Lagrangian method to deep neural networks. We believe that wild data has the potential to dramatically advance OOD detection in practice, thereby helping to accelerate the deployment of safe and reliable machine learning.  

\section*{Acknowledgments}
We would like to acknowledge Xuefeng Du's help in verifying the baseline performance. We would also like to thank Ahmet Alacaoglu for very useful conversations. %\SL{Rob would you like to add this?}

\bibliography{refs}

\begin{thebibliography}{47}
\providecommand{\natexlab}[1]{#1}
\providecommand{\url}[1]{\texttt{#1}}
\expandafter\ifx\csname urlstyle\endcsname\relax
  \providecommand{\doi}[1]{doi: #1}\else
  \providecommand{\doi}{doi: \begingroup \urlstyle{rm}\Url}\fi

\bibitem[Bendale \& Boult(2015)Bendale and Boult]{openworld}
Bendale, A. and Boult, T.
\newblock Towards open world recognition.
\newblock In \emph{Proceedings of the IEEE conference on computer vision and
  pattern recognition}, pp.\  1893--1902, 2015.

\bibitem[Bevandi{\'c} et~al.(2018)Bevandi{\'c}, Kre{\v{s}}o, Or{\v{s}}i{\'c},
  and {\v{S}}egvi{\'c}]{bevandic2018discriminative}
Bevandi{\'c}, P., Kre{\v{s}}o, I., Or{\v{s}}i{\'c}, M., and {\v{S}}egvi{\'c},
  S.
\newblock Discriminative out-of-distribution detection for semantic
  segmentation.
\newblock \emph{arXiv preprint arXiv:1808.07703}, 2018.

\bibitem[Blanchard et~al.(2010)Blanchard, Lee, and Scott]{blanchard2010semi}
Blanchard, G., Lee, G., and Scott, C.
\newblock Semi-supervised novelty detection.
\newblock \emph{The Journal of Machine Learning Research}, 11:\penalty0
  2973--3009, 2010.

\bibitem[Chalapathy \& Chawla(2019)Chalapathy and Chawla]{chalapathy2019deep}
Chalapathy, R. and Chawla, S.
\newblock Deep learning for anomaly detection: A survey.
\newblock \emph{arXiv preprint arXiv:1901.03407}, 2019.

\bibitem[Chalapathy et~al.(2018)Chalapathy, Menon, and
  Chawla]{chalapathy2018anomaly}
Chalapathy, R., Menon, A.~K., and Chawla, S.
\newblock Anomaly detection using one-class neural networks.
\newblock \emph{arXiv preprint arXiv:1802.06360}, 2018.

\bibitem[Chen et~al.(2021)Chen, Li, Wu, Liang, and Jha]{chen2020informative}
Chen, J., Li, Y., Wu, X., Liang, Y., and Jha, S.
\newblock Atom: Robustifying out-of-distribution detection using outlier
  mining.
\newblock \emph{In Proceedings of European Conference on Machine Learning and
  Principles and Practice of Knowledge Discovery in Databases (ECML PKDD)},
  2021.

\bibitem[Cimpoi et~al.(2014)Cimpoi, Maji, Kokkinos, Mohamed, and
  Vedaldi]{cimpoi_describing_2014}
Cimpoi, M., Maji, S., Kokkinos, I., Mohamed, S., and Vedaldi, a.~A.
\newblock Describing {Textures} in the {Wild}.
\newblock In \emph{Proceedings of the {IEEE} {Conf}. on {Computer} {Vision} and
  {Pattern} {Recognition} ({CVPR})}, 2014.

\bibitem[Daniel et~al.(2019)Daniel, Kurutach, and Tamar]{daniel2019deep}
Daniel, T., Kurutach, T., and Tamar, A.
\newblock Deep variational semi-supervised novelty detection.
\newblock \emph{arXiv preprint arXiv:1911.04971}, 2019.

\bibitem[Du et~al.(2022{\natexlab{a}})Du, Wang, Gozum, and Li]{du2022unknown}
Du, X., Wang, X., Gozum, G., and Li, Y.
\newblock Unknown-aware object detection: Learning what you don’t know from
  videos in the wild.
\newblock In \emph{Proceedings of the IEEE/CVF Conference on Computer Vision
  and Pattern Recognition}, 2022{\natexlab{a}}.

\bibitem[Du et~al.(2022{\natexlab{b}})Du, Wang, Cai, and Li]{du2022vos}
Du, X., Wang, Z., Cai, M., and Li, Y.
\newblock Vos: Learning what you don’t know by virtual outlier synthesis.
\newblock \emph{Proceedings of the International Conference on Learning
  Representations}, 2022{\natexlab{b}}.

\bibitem[Duchi et~al.(2011)Duchi, Hazan, and Singer]{duchi2011adaptive}
Duchi, J., Hazan, E., and Singer, Y.
\newblock Adaptive subgradient methods for online learning and stochastic
  optimization.
\newblock \emph{Journal of machine learning research}, 12\penalty0 (7), 2011.

\bibitem[Ergen \& Kozat(2019)Ergen and Kozat]{ergen2019unsupervised}
Ergen, T. and Kozat, S.~S.
\newblock Unsupervised anomaly detection with lstm neural networks.
\newblock \emph{IEEE transactions on neural networks and learning systems},
  31\penalty0 (8):\penalty0 3127--3141, 2019.

\bibitem[Hendrycks \& Gimpel(2017)Hendrycks and Gimpel]{Kevin}
Hendrycks, D. and Gimpel, K.
\newblock A baseline for detecting misclassified and out-of-distribution
  examples in neural networks.
\newblock \emph{Proceedings of International Conference on Learning
  Representations}, 2017.

\bibitem[Hendrycks et~al.(2019)Hendrycks, Mazeika, and
  Dietterich]{hendrycks2018deep}
Hendrycks, D., Mazeika, M., and Dietterich, T.
\newblock Deep anomaly detection with outlier exposure.
\newblock In \emph{International Conference on Learning Representations}, 2019.

\bibitem[Hestenes(1969)]{hestenes1969multiplier}
Hestenes, M.~R.
\newblock Multiplier and gradient methods.
\newblock \emph{Journal of optimization theory and applications}, 4\penalty0
  (5):\penalty0 303--320, 1969.

\bibitem[Hsu et~al.(2020)Hsu, Shen, Jin, and Kira]{hsu2020generalized}
Hsu, Y.-C., Shen, Y., Jin, H., and Kira, Z.
\newblock Generalized odin: Detecting out-of-distribution image without
  learning from out-of-distribution data.
\newblock In \emph{Proceedings of the IEEE/CVF Conference on Computer Vision
  and Pattern Recognition}, pp.\  10951--10960, 2020.

\bibitem[Huang et~al.(2021)Huang, Geng, and Li]{huang2021importance}
Huang, R., Geng, A., and Li, Y.
\newblock On the importance of gradients for detecting distributional shifts in
  the wild.
\newblock In \emph{Advances in Neural Information Processing Systems}, 2021.

\bibitem[Huber(1964)]{huber1964}
Huber, P.~J.
\newblock Robust estimation of a location parameter.
\newblock \emph{Annals of Mathematical Statistics}, 35:\penalty0 73--101, March
  1964.

\bibitem[Krizhevsky et~al.(2009)Krizhevsky, Hinton, and
  {others}]{krizhevsky_learning_2009}
Krizhevsky, A., Hinton, G., and {others}.
\newblock Learning multiple layers of features from tiny images.
\newblock 2009.
\newblock Publisher: Citeseer.

\bibitem[Lee et~al.(2018{\natexlab{a}})Lee, Lee, Lee, and Shin]{lee2018gan}
Lee, K., Lee, H., Lee, K., and Shin, J.
\newblock Training confidence-calibrated classifiers for detecting
  out-of-distribution samples.
\newblock \emph{International Conference on Learning Representations (ICLR)},
  2018{\natexlab{a}}.

\bibitem[Lee et~al.(2018{\natexlab{b}})Lee, Lee, Lee, and Shin]{lee2018simple}
Lee, K., Lee, K., Lee, H., and Shin, J.
\newblock A simple unified framework for detecting out-of-distribution samples
  and adversarial attacks.
\newblock In \emph{Advances in Neural Information Processing Systems}, pp.\
  7167--7177, 2018{\natexlab{b}}.

\bibitem[Liang et~al.(2018)Liang, Li, and Srikant]{liang2018enhancing}
Liang, S., Li, Y., and Srikant, R.
\newblock Enhancing the reliability of out-of-distribution image detection in
  neural networks.
\newblock In \emph{6th International Conference on Learning Representations,
  ICLR 2018}, 2018.

\bibitem[Liu et~al.(2020)Liu, Wang, Owens, and Li]{liu2020energy}
Liu, W., Wang, X., Owens, J., and Li, Y.
\newblock Energy-based out-of-distribution detection.
\newblock \emph{Advances in Neural Information Processing Systems}, 2020.

\bibitem[Malinin \& Gales(2018)Malinin and Gales]{malinin2018predictive}
Malinin, A. and Gales, M.
\newblock Predictive uncertainty estimation via prior networks.
\newblock \emph{arXiv preprint arXiv:1802.10501}, 2018.

\bibitem[Ming et~al.(2022)Ming, Fan, and Li]{ming2022posterior}
Ming, Y., Fan, Y., and Li, Y.
\newblock Poem: Out-of-distribution detection with posterior sampling.
\newblock In \emph{International Conference on Machine Learning (ICML)}. PMLR,
  2022.

\bibitem[Morteza \& Li(2022)Morteza and Li]{morteza2022provable}
Morteza, P. and Li, Y.
\newblock Provable guarantees for understanding out-of-distribution detection.
\newblock In \emph{Proceedings of the AAAI Conference on Artificial
  Intelligence}, 2022.

\bibitem[Netzer et~al.(2011)Netzer, Wang, Coates, Bissacco, Wu, and
  Ng]{netzer_reading_2011}
Netzer, Y., Wang, T., Coates, A., Bissacco, A., Wu, B., and Ng, A.~Y.
\newblock Reading {Digits} in {Natural} {Images} with {Unsupervised} {Feature}
  {Learning}.
\newblock In \emph{{NIPS} {Workshop} on {Deep} {Learning} and {Unsupervised}
  {Feature} {Learning} 2011}, 2011.

\bibitem[Nguyen et~al.(2015)Nguyen, Yosinski, and Clune]{fool}
Nguyen, A., Yosinski, J., and Clune, J.
\newblock Deep neural networks are easily fooled: High confidence predictions
  for unrecognizable images.
\newblock In \emph{Proceedings of the IEEE conference on computer vision and
  pattern recognition}, pp.\  427--436, 2015.

\bibitem[Nocedal \& Wright(2006)Nocedal and Wright]{nocedal2006numerical}
Nocedal, J. and Wright, S.
\newblock \emph{Numerical optimization}.
\newblock Springer Science \& Business Media, 2006.

\bibitem[Perera \& Patel(2019)Perera and Patel]{perera2019learning}
Perera, P. and Patel, V.~M.
\newblock Learning deep features for one-class classification.
\newblock \emph{IEEE Transactions on Image Processing}, 28\penalty0
  (11):\penalty0 5450--5463, 2019.

\bibitem[Rockafellar(1973)]{rockafellar1973dual}
Rockafellar, R.~T.
\newblock A dual approach to solving nonlinear programming problems by
  unconstrained optimization.
\newblock \emph{Mathematical programming}, 5\penalty0 (1):\penalty0 354--373,
  1973.

\bibitem[Ruff et~al.(2018)Ruff, Vandermeulen, Goernitz, Deecke, Siddiqui,
  Binder, M{\"u}ller, and Kloft]{ruff2018deep}
Ruff, L., Vandermeulen, R., Goernitz, N., Deecke, L., Siddiqui, S.~A., Binder,
  A., M{\"u}ller, E., and Kloft, M.
\newblock Deep one-class classification.
\newblock In \emph{International conference on machine learning}, pp.\
  4393--4402. PMLR, 2018.

\bibitem[Ruff et~al.(2020)Ruff, Vandermeulen, G{\"o}rnitz, Binder, M{\"u}ller,
  M{\"u}ller, and Kloft]{ruff2019deep}
Ruff, L., Vandermeulen, R.~A., G{\"o}rnitz, N., Binder, A., M{\"u}ller, E.,
  M{\"u}ller, K.-R., and Kloft, M.
\newblock Deep semi-supervised anomaly detection.
\newblock In \emph{International Conference on Learning Representations}, 2020.

\bibitem[Ruff et~al.(2021)Ruff, Kauffmann, Vandermeulen, Montavon, Samek,
  Kloft, Dietterich, and M{\"u}ller]{ruff2021unifying}
Ruff, L., Kauffmann, J.~R., Vandermeulen, R.~A., Montavon, G., Samek, W.,
  Kloft, M., Dietterich, T.~G., and M{\"u}ller, K.-R.
\newblock A unifying review of deep and shallow anomaly detection.
\newblock \emph{Proceedings of the IEEE}, 2021.

\bibitem[Sangalli et~al.(2021)Sangalli, Erdil, H{\"o}tker, Donati, and
  Konukoglu]{sangalli2021constrained}
Sangalli, S., Erdil, E., H{\"o}tker, A., Donati, O., and Konukoglu, E.
\newblock Constrained optimization to train neural networks on critical and
  under-represented classes.
\newblock \emph{Advances in Neural Information Processing Systems}, 34, 2021.

\bibitem[Song et~al.(2017)Song, Jiang, Men, and Yang]{song2017hybrid}
Song, H., Jiang, Z., Men, A., and Yang, B.
\newblock A hybrid semi-supervised anomaly detection model for high-dimensional
  data.
\newblock \emph{Computational intelligence and neuroscience}, 2017, 2017.

\bibitem[Sun et~al.(2021)Sun, Guo, and Li]{sun2021react}
Sun, Y., Guo, C., and Li, Y.
\newblock React: Out-of-distribution detection with rectified activations.
\newblock In \emph{Advances in Neural Information Processing Systems}, 2021.

\bibitem[Sun et~al.(2022)Sun, Ming, Zhu, and Li]{sun2022knn}
Sun, Y., Ming, Y., Zhu, X., and Li, Y.
\newblock Out-of-distribution detection with deep nearest neighbors.
\newblock In \emph{International Conference on Machine Learning (ICML)}. PMLR,
  2022.

\bibitem[Tack et~al.(2020)Tack, Mo, Jeong, and Shin]{tack2020csi}
Tack, J., Mo, S., Jeong, J., and Shin, J.
\newblock Csi: Novelty detection via contrastive learning on distributionally
  shifted instances.
\newblock In \emph{Advances in Neural Information Processing Systems}, 2020.

\bibitem[Wang et~al.(2021)Wang, Liu, Bocchieri, and Li]{wang2021can}
Wang, H., Liu, W., Bocchieri, A., and Li, Y.
\newblock Can multi-label classification networks know what they don’t know?
\newblock \emph{Advances in Neural Information Processing Systems}, 34, 2021.

\bibitem[Xu(2021{\natexlab{a}})]{xu2021first}
Xu, Y.
\newblock First-order methods for constrained convex programming based on
  linearized augmented lagrangian function.
\newblock \emph{Informs Journal on Optimization}, 3\penalty0 (1):\penalty0
  89--117, 2021{\natexlab{a}}.

\bibitem[Xu(2021{\natexlab{b}})]{xu2021iteration}
Xu, Y.
\newblock Iteration complexity of inexact augmented lagrangian methods for
  constrained convex programming.
\newblock \emph{Mathematical Programming}, 185\penalty0 (1):\penalty0 199--244,
  2021{\natexlab{b}}.

\bibitem[Yan \& Xu(2020)Yan and Xu]{yan2020adaptive}
Yan, Y. and Xu, Y.
\newblock Adaptive primal-dual stochastic gradient method for
  expectation-constrained convex stochastic programs.
\newblock \emph{arXiv preprint arXiv:2012.14943}, 2020.

\bibitem[Yu et~al.(2016)Yu, Seff, Zhang, Song, Funkhouser, and
  Xiao]{yu_lsun_2016}
Yu, F., Seff, A., Zhang, Y., Song, S., Funkhouser, T., and Xiao, J.
\newblock {LSUN}: {Construction} of a {Large}-scale {Image} {Dataset} using
  {Deep} {Learning} with {Humans} in the {Loop}.
\newblock \emph{arXiv:1506.03365 [cs]}, June 2016.
\newblock arXiv: 1506.03365.

\bibitem[Zagoruyko \& Komodakis(2016)Zagoruyko and
  Komodakis]{zagoruyko_wide_2016}
Zagoruyko, S. and Komodakis, N.
\newblock Wide {Residual} {Networks}.
\newblock In \emph{Procedings of the {British} {Machine} {Vision}
  {Conference}}, 2016.

\bibitem[Zhang et~al.(2021)Zhang, Bengio, Hardt, Recht, and
  Vinyals]{zhang2021understanding}
Zhang, C., Bengio, S., Hardt, M., Recht, B., and Vinyals, O.
\newblock Understanding deep learning (still) requires rethinking
  generalization.
\newblock \emph{Communications of the ACM}, 64\penalty0 (3):\penalty0 107--115,
  2021.

\bibitem[Zhou et~al.(2018)Zhou, Lapedriza, Khosla, Oliva, and
  Torralba]{zhou_places_2018}
Zhou, B., Lapedriza, A., Khosla, A., Oliva, A., and Torralba, A.
\newblock Places: {A} 10 {Million} {Image} {Database} for {Scene}
  {Recognition}.
\newblock \emph{IEEE Transactions on Pattern Analysis and Machine
  Intelligence}, 40\penalty0 (6):\penalty0 1452--1464, June 2018.

\end{thebibliography}
\bibliographystyle{icml2022}

%%%%%%%%%%%%%%%%%%%%%%%%%%%%%%%%%%%%%%%%%%%%%%%%%%%%%%%%%%%%%%%%%%%%%%%%%%%%%%%
%%%%%%%%%%%%%%%%%%%%%%%%%%%%%%%%%%%%%%%%%%%%%%%%%%%%%%%%%%%%%%%%%%%%%%%%%%%%%%%
% APPENDIX
%%%%%%%%%%%%%%%%%%%%%%%%%%%%%%%%%%%%%%%%%%%%%%%%%%%%%%%%%%%%%%%%%%%%%%%%%%%%%%%
%%%%%%%%%%%%%%%%%%%%%%%%%%%%%%%%%%%%%%%%%%%%%%%%%%%%%%%%%%%%%%%%%%%%%%%%%%%%%%%
\newpage
\appendix
\onecolumn

\section{Additional Experimental Results}

\subsection{Main Experiments: \texttt{CIFAR-10}}\label{sec:appendix_CIFAR10}
Table~\ref{tab:cifar_main} shows a comparison of our method's performance with OOD baseline methods on \texttt{CIFAR-10}. On average, \texttt{WOODS} outperforms \texttt{CSI}, the best of the baseline methods, by \textbf{13.59\%}, and \texttt{OE} by \textbf{3.61\%}. Since all of the methods that use $\mathbb{P}_\textrm{wild}$ achieve very good performance on most of the auxiliary test datasets, we highlight in particular the results for \texttt{Textures} (where our method outperforms \texttt{CSI} by \textbf{14.50\%} and \texttt{OE} by \textbf{4.47\%}) and \texttt{Places365} (where our method outperforms \texttt{CSI} by \textbf{22.46\%} and \texttt{OE} by \textbf{10.98\%}). Table~\ref{tab:pi} shows the results of ablation on $\pi$ using \texttt{CIFAR-10} as the ID dataset, where \texttt{WOODS} demonstrates similarly strong performance across the board.

Table~\ref{tab:pi} shows the results of ablation on $\pi$ using \texttt{CIFAR-10} as the ID dataset. In general, \texttt{WOODS} outperforms OE and Energy by an even larger margin than on CIFAR-10 and across many values of $\pi$.

\begin{table*}[h]
\centering
\scalebox{0.64}{
\begin{tabular}{llllllllllllllll} \toprule
\multicolumn{1}{c}{\multirow{3}{*}{\textbf{Method}}} & \multicolumn{10}{c}{\textbf{OOD Dataset}} & \multicolumn{2}{c}{\multirow{2}{*}{\textbf{Average}}} & \multirow{3}{*}{\textbf{Acc.}} \\
\multicolumn{1}{c}{} &
\multicolumn{2}{c}{\textbf{SVHN}} & 
\multicolumn{2}{c}{\textbf{LSUN-R}} & 
\multicolumn{2}{c}{\textbf{LSUN-C}} &
\multicolumn{2}{c}{\textbf{Textures}} & \multicolumn{2}{c}{\textbf{Places365}} & \multicolumn{2}{c}{} & \\
\multicolumn{1}{c}{} & {FPR$\downarrow$} & {AUROC$\uparrow$} & {FPR$\downarrow$} & {AUROC$\uparrow$} & {FPR$\downarrow$} & {AUROC$\uparrow$} & {FPR$\downarrow$} & {AUROC$\uparrow$} & {FPR$\downarrow$} & {AUROC$\uparrow$} & {FPR$\downarrow$} & {AUROC$\uparrow$} & \\ \midrule  &\multicolumn{13}{c}{\textbf{With $\Pin$ only}}          \\
MSP & 48.49 & 91.89 & 52.15 & 91.37 & 30.80 & 95.65 & 59.28 & 88.50 & 59.48 & 88.20 & 50.04 & 91.12  & \textbf{94.84} \\
ODIN & 33.35 & 91.96 & 26.62 & 94.57 & 15.52 & 97.04 & 49.12 & 84.97 & 57.40 & 84.49 & 36.40 & 90.61 & \textbf{94.84} \\
Energy & 35.59 & 90.96 & 27.58 & 94.24 & 8.26 & 98.35 & 52.79 & 85.22 & 40.14 & 89.89 & 32.87 & 91.73 & \textbf{94.84} \\
Mahalanobis & 12.89 & 97.62 & 42.62 & 93.23 & 39.22 & 94.15 & 15.00 & 97.33 & 68.57 & 84.61 & 35.66 & 93.34 & \textbf{94.84}\\ 
GODIN & 13.55 & 97.61 & 17.93 & 96.86 & 17.68 & 96.93 & 29.43 & 94.87 & 41.27 & 91.49 & 23.97 & 95.55 & 94.48\\
CSI &  17.30 & 97.40 & 12.15 & 98.01 & 1.95 & 99.55 & 20.45 & 95.93 & 34.95 & 93.64 & 17.36 & 96.91 & 94.17 \\
\midrule &\multicolumn{13}{c}{\textbf{With $\Pin$ and $\Pwild$}} \\
OE                       &           0.85$^{\pm 0.1}$ &           99.82$^{\pm 0.0}$ &           0.33$^{\pm 0.0}$ &           99.93$^{\pm 0.0}$ &           1.84$^{\pm 0.2}$ &           99.65$^{\pm 0.0}$ &            10.42$^{\pm 0.6}$ &           98.01$^{\pm 0.0}$ &           23.47$^{\pm 0.5}$ &           94.62$^{\pm 0.0}$ & 7.38$^{\pm 0.3}$ & 98.41$^{\pm 0.0}$ &           94.07$^{\pm 0.2}$ \\
Energy (w/ OE)           &           4.95$^{\pm 4.5}$ &           98.92$^{\pm 1.0}$ &           5.04$^{\pm 4.9}$ &           98.83$^{\pm 1.1}$ &           1.93$^{\pm 1.4}$ &           99.49$^{\pm 0.3}$ &           13.43$^{\pm 6.5}$ &           96.69$^{\pm 1.8}$ &           17.26$^{\pm 0.4}$ &           95.84$^{\pm 0.1}$ &      8.52$^{\pm 3.5}$      &        97.95$^{\pm 0.9}$    &           94.81$^{\pm 0.1}$ \\
WOODS (ours)             &           0.15$^{\pm 0.0}$ &  \textbf{99.97}$^{\pm 0.0}$ &           0.03$^{\pm 0.0}$ &           \textbf{99.99}$^{\pm 0.0}$ &           0.22$^{\pm 0.0}$ &           99.94$^{\pm 0.0}$ &           5.95$^{\pm 0.6}$ &           98.79$^{\pm 0.1}$ &           12.49$^{\pm 0.3}$ &  \textbf{97.00}$^{\pm 0.0}$ &  3.77$^{\pm 0.2}$          & \textbf{99.14}$^{\pm 0.0}$  &  \textbf{94.84}$^{\pm 0.1}$ \\
WOODS-alt (ours) &  \textbf{0.10}$^{\pm 0.0}$ &           99.96$^{\pm 0.0}$ &  \textbf{0.02}$^{\pm 0.0}$ &  \textbf{99.99}$^{\pm 0.0}$ &  \textbf{0.08}$^{\pm 0.0}$ &  \textbf{99.96}$^{\pm 0.0}$ &  \textbf{2.11}$^{\pm 0.2}$ &  \textbf{99.32}$^{\pm 0.0}$ &  \textbf{11.51}$^{\pm 0.2}$ &           96.25$^{\pm 0.2}$ &  \textbf{2.76}$^{\pm 0.0}$ &     99.10$^{\pm 0.0}$       &           94.71$^{\pm 0.1}$ \\
\bottomrule
\end{tabular}
}
\caption{\small \textbf{Main results when $\P_\text{out}^\text{test} = \P_\text{out}$.} Comparison with competitive OOD detection methods on \texttt{CIFAR-10}. For methods using $\Pwild$, we train under the same dataset and same $\pi=0.1$. $\uparrow$ indicates larger values are better and vice versa. $\pm x$ denotes the standard error, rounded to the first decimal point.} 
\label{tab:cifar_main}
\end{table*}

\begin{table*}[h]
\centering
\scalebox{0.64}{
\begin{tabular}{llllllllllllllll} \toprule
\multicolumn{1}{c}{\multirow{3}{*}{\textbf{Method}}} & \multicolumn{12}{c}{\textbf{OOD Dataset}} & \multirow{3}{*}{\textbf{Acc.}} \\
\multicolumn{1}{c}{} &
\multicolumn{2}{c}{\textbf{SVHN}} & 
\multicolumn{2}{c}{\textbf{LSUN-R}} & 
\multicolumn{2}{c}{\textbf{LSUN-C}} &
\multicolumn{2}{c}{\textbf{Textures}} & \multicolumn{2}{c}{\textbf{Places365}} & \multicolumn{2}{c}{\textbf{300K Rand. Img.}} & \\
\multicolumn{1}{c}{} & {FPR$\downarrow$} & {AUROC$\uparrow$} & {FPR$\downarrow$} & {AUROC$\uparrow$} & {FPR$\downarrow$} & {AUROC$\uparrow$} & {FPR$\downarrow$} & {AUROC$\uparrow$} & {FPR$\downarrow$} & {AUROC$\uparrow$} & {FPR$\downarrow$} & {AUROC$\uparrow$} & \\ \midrule &\multicolumn{13}{c}{{$\pi=0.05$}}          \\
OE             &          17.23$^{\pm 1.7}$ &           96.66$^{\pm 0.3}$ &          18.31$^{\pm 1.0}$ &           96.80$^{\pm 0.2}$ &           7.38$^{\pm 0.4}$ &           98.59$^{\pm 0.1}$ &          29.47$^{\pm 1.1}$ &           93.63$^{\pm 0.3}$ &           34.47$^{\pm 0.4}$ &           92.44$^{\pm 0.2}$ &     33.69$^{\pm 0.3}$ &           91.89$^{\pm 0.2}$ &           94.38$^{\pm 0.1}$ \\
Energy (w/ OE) &          14.19$^{\pm 2.0}$ &           97.04$^{\pm 0.5}$ &          12.54$^{\pm 0.1}$ &           97.47$^{\pm 0.0}$ &           3.76$^{\pm 0.2}$ &           99.16$^{\pm 0.0}$ &          29.44$^{\pm 1.6}$ &           92.73$^{\pm 0.5}$ &           27.42$^{\pm 0.3}$ &           93.21$^{\pm 0.1}$ &            29.34$^{\pm 1.2}$ &           92.23$^{\pm 0.5}$ &           94.71$^{\pm 0.1}$ \\
WOODS (ours)   &  \textbf{9.23}$^{\pm 1.0}$ &  \textbf{97.98}$^{\pm 0.2}$ &  \textbf{6.89}$^{\pm 0.6}$ &  \textbf{98.53}$^{\pm 0.1}$ &  \textbf{1.82}$^{\pm 0.1}$ &  \textbf{99.56}$^{\pm 0.0}$ &  \textbf{23.12}$^{\pm 1.5}$ &  \textbf{94.98}$^{\pm 0.3}$ &  \textbf{21.58}$^{\pm 0.2}$ &  \textbf{95.07}$^{\pm 0.1}$ &   \textbf{24.39}$^{\pm 0.9}$ &  \textbf{94.16}$^{\pm 0.4}$  & \textbf{94.83}$^{\pm 0.1}$ \\
\midrule &\multicolumn{13}{c}{$\pi=0.1$}          \\
OE             &          13.18$^{\pm 1.3}$ &           97.34$^{\pm 0.2}$ &          14.32$^{\pm 0.4}$ &           97.44$^{\pm 0.1}$ &           5.87$^{\pm 0.2}$ &           98.86$^{\pm 0.0}$ &          25.69$^{\pm 0.6}$ &           94.35$^{\pm 0.1}$ &           30.54$^{\pm 0.3}$ &           93.31$^{\pm 0.1}$ &           30.69$^{\pm 0.7}$ &           92.80$^{\pm 0.2}$ &           94.21$^{\pm 0.1}$ \\
Energy (w/ OE) &           8.52$^{\pm 1.8}$ &           98.13$^{\pm 0.3}$ &           9.05$^{\pm 0.7}$ &           98.13$^{\pm 0.1}$ &           2.78$^{\pm 0.2}$ &           99.38$^{\pm 0.0}$ &          22.32$^{\pm 1.3}$ &           94.72$^{\pm 0.4}$ &           23.74$^{\pm 0.3}$ &           94.26$^{\pm 0.1}$ &           24.59$^{\pm 0.8}$ &           93.99$^{\pm 0.2}$ &           94.54$^{\pm 0.1}$ \\
WOODS (ours)   &  \textbf{5.70}$^{\pm 1.0}$ &  \textbf{98.54}$^{\pm 0.2}$ &  \textbf{4.13}$^{\pm 0.4}$ &  \textbf{99.01}$^{\pm 0.1}$ &  \textbf{1.31}$^{\pm 0.1}$ &  \textbf{99.66}$^{\pm 0.0}$ &  \textbf{17.92}$^{\pm 1.3}$ &  \textbf{96.43}$^{\pm 0.2}$ &  \textbf{19.14}$^{\pm 0.3}$ &  \textbf{95.74}$^{\pm 0.0}$ &  \textbf{19.82}$^{\pm 0.5}$ &  \textbf{95.52}$^{\pm 0.1}$ &  \textbf{94.74}$^{\pm 0.0}$ \\
\midrule &\multicolumn{13}{c}{$\pi=0.2$}          \\
OE             &           9.94$^{\pm 1.0}$ &           98.04$^{\pm 0.1}$ &          10.37$^{\pm 0.5}$ &           98.10$^{\pm 0.1}$ &           4.82$^{\pm 0.3}$ &           99.08$^{\pm 0.0}$ &          20.34$^{\pm 0.7}$ &           95.67$^{\pm 0.2}$ &           27.08$^{\pm 0.5}$ &           93.98$^{\pm 0.2}$ &           24.94$^{\pm 0.8}$ &           94.25$^{\pm 0.1}$ &           94.10$^{\pm 0.0}$ \\
Energy (w/ OE) &           5.88$^{\pm 1.1}$ &           98.50$^{\pm 0.2}$ &           6.18$^{\pm 0.7}$ &           98.65$^{\pm 0.1}$ &           2.38$^{\pm 0.1}$ &           99.45$^{\pm 0.0}$ &           17.28$^{\pm 0.8}$ &           96.14$^{\pm 0.3}$ &           21.28$^{\pm 0.2}$ &           94.93$^{\pm 0.1}$ &            19.54$^{\pm 0.6}$ &           94.99$^{\pm 0.1}$ &           94.30$^{\pm 0.1}$ \\
WOODS (ours)   &  \textbf{4.20}$^{\pm 0.8}$ &  \textbf{98.68}$^{\pm 0.2}$ &  \textbf{3.35}$^{\pm 0.5}$ &  \textbf{99.19}$^{\pm 0.1}$ &  \textbf{1.31}$^{\pm 0.1}$ &  \textbf{99.64}$^{\pm 0.0}$ &  \textbf{13.17}$^{\pm 0.7}$ &  \textbf{97.45}$^{\pm 0.1}$ &  \textbf{17.58}$^{\pm 0.3}$ &  \textbf{96.12}$^{\pm 0.0}$ &   \textbf{15.98}$^{\pm 0.4}$ &  \textbf{96.29}$^{\pm 0.1}$ &  \textbf{94.75}$^{\pm 0.1}$ \\
\midrule &\multicolumn{13}{c}{$\pi=0.5$}          \\
OE             &           5.07$^{\pm 1.1}$ &           98.75$^{\pm 0.1}$ &           5.23$^{\pm 0.8}$ &           98.88$^{\pm 0.1}$ &           2.94$^{\pm 0.2}$ &           99.36$^{\pm 0.0}$ &           12.06$^{\pm 0.7}$ &           97.45$^{\pm 0.2}$ &           19.62$^{\pm 0.2}$ &           95.47$^{\pm 0.1}$ &            16.32$^{\pm 0.6}$ &           95.90$^{\pm 0.0}$ &           94.16$^{\pm 0.1}$ \\
Energy (w/ OE) &  \textbf{2.57}$^{\pm 0.4}$ &  \textbf{99.16}$^{\pm 0.1}$ &           4.64$^{\pm 0.7}$ &           98.94$^{\pm 0.1}$ &           1.57$^{\pm 0.1}$ &  \textbf{99.60}$^{\pm 0.0}$ &       11.77$^{\pm 0.9}$ &           97.50$^{\pm 0.2}$ &           16.63$^{\pm 0.3}$ &           96.04$^{\pm 0.0}$ &          13.86$^{\pm 0.1}$ &           96.34$^{\pm 0.1}$ &           94.57$^{\pm 0.1}$ \\
WOODS (ours)   &           3.58$^{\pm 1.4}$ &           98.69$^{\pm 0.2}$ &  \textbf{2.78}$^{\pm 0.2}$ &  \textbf{99.21}$^{\pm 0.0}$ &  \textbf{1.44}$^{\pm 0.1}$ &           99.56$^{\pm 0.0}$ &  \textbf{9.20}$^{\pm 0.3}$ &  \textbf{98.09}$^{\pm 0.1}$ &  \textbf{15.25}$^{\pm 0.1}$ &  \textbf{96.47}$^{\pm 0.0}$ &   \textbf{11.82}$^{\pm 0.1}$ &  \textbf{96.84}$^{\pm 0.1}$ &  \textbf{94.72}$^{\pm 0.0}$ \\
\midrule &\multicolumn{13}{c}{$\pi=1.0$}          \\
OE             &  \textbf{1.72}$^{\pm 0.1}$ &  \textbf{99.21}$^{\pm 0.1}$ &           2.81$^{\pm 0.2}$ &  \textbf{99.20}$^{\pm 0.0}$ &           1.72$^{\pm 0.1}$ &           99.50$^{\pm 0.0}$ &           6.70$^{\pm 0.3}$ &  \textbf{98.51}$^{\pm 0.0}$ &           12.43$^{\pm 0.1}$ &           96.98$^{\pm 0.0}$ &  10.43$^{\pm 0.4}$ &           97.32$^{\pm 0.0}$ &           94.62$^{\pm 0.1}$ \\
Energy (w/ OE) &           8.20$^{\pm 2.4}$ &           97.85$^{\pm 0.4}$ &           2.72$^{\pm 0.2}$ &           99.05$^{\pm 0.1}$ &           1.62$^{\pm 0.1}$ &           99.36$^{\pm 0.0}$ &           \textbf{5.77}$^{\pm 0.2}$ &           98.50$^{\pm 0.1}$ &  \textbf{11.29}$^{\pm 0.2}$ &  \textbf{97.09}$^{\pm 0.0}$ &      \textbf{9.82}$^{\pm 0.1}$ &  \textbf{97.33}$^{\pm 0.0}$ &           94.23$^{\pm 0.1}$ \\
WOODS (ours)   &           2.35$^{\pm 0.5}$ &           98.42$^{\pm 0.2}$ &  \textbf{2.27}$^{\pm 0.2}$ &           99.14$^{\pm 0.1}$ &  \textbf{1.34}$^{\pm 0.0}$ &  \textbf{99.50}$^{\pm 0.0}$ &  6.31$^{\pm 0.4}$ &           98.42$^{\pm 0.1}$ &           12.50$^{\pm 0.2}$ &           96.78$^{\pm 0.0}$ &   10.01$^{\pm 0.3}$ &           97.11$^{\pm 0.1}$ &  \textbf{94.85}$^{\pm 0.0}$ \\
 \bottomrule
\end{tabular}
}
\caption{\small \textbf{Effect of $\pi$.} A larger $\pi$ indicates more OOD data in the mixture distribution $\Pwild$. ID dataset is \texttt{CIFAR-10}, and the auxiliary outlier training data is \texttt{300K Random Images}.  $\uparrow$ indicates larger values are better and vice versa. $\pm x$ denotes the standard error, rounded to the first decimal point.}
\label{tab:pi}
\end{table*}

\subsection{\texttt{WOODS-alt} Description}
\label{sec:results_ssnd_setting}

Here, we describe \texttt{WOODS-alt}, for which an OOD confidence score is not extracted directly from the output of the ID classifier, but rather learned by a separate neural network attached to the ID classifier's penultimate layer. The additional neural network has one fully-connected hidden layer with 300 neurons, followed by a ReLU activation and a single output logit, which provides an OOD confidence score, denoted $g_\theta(\cdot)$. With this architecture, we apply the same ALM algorithm to solve the following constrained optimization problem:
\begin{align*}
 \inf_{\theta} & \frac{1}{m}\sum_{i=1}^{m} \max(1- g_\theta(\tilde{\*x}_i)), 0)\\
         \text{s.t. } & \frac{1}{n}\sum_{j=1}^n \max(1+ g_\theta(\*x_j)), 0)   \leq \alpha \nonumber \\
        & \frac{1}{n} \sum_{j=1}^n \L_\text{cls}( f_{\theta}(\*x_j) , y_j) \leq \tau. \nonumber
\end{align*}

\subsection{Additional Experimental Details}\label{sec:exper_details}

Here we give additional experimental details, presented in Sections \ref{experiments} and \ref{sec:appendix_CIFAR10}. \texttt{Energy} and \texttt{OE} both optimize an objective of the form
\begin{align*}
\min_\theta \mathcal{L}_\text{cls} + \lambda \mathcal{L}_\text{OOD}
\end{align*}
For energy, we varied $\lambda \in \{0.1,1,5\}$ and for OE we varied $\lambda \in \{0.1,0.5,1\}$. 

We simulate the mixture distribution as follows. For each iteration at training, for the ID dataset we draw one batch of size $128$ and for the wild dataset $\P_\text{wild}$ we draw another batch of size $128$ where each example is drawn from $\P_\text{out}$ with probability $\pi$ and from $\P_\text{in}$ with probability $1-\pi$.

We split the data as follows. For the stationary setting $\P_\text{out}^\text{test} = \P_\text{out}$ (say from \texttt{SVHN}), we use $70\%$ of the \texttt{SVHN} for the mixture training dataset and for the validation dataset . We use the remaining examples for the test set. Of the data for training/validation, we use $30\%$ for validation and the remaining for training. In the nonstationary setting $\P_\text{out}^\text{test} \neq \P_\text{out}$, we use the same splitting approach except in the initial split into train/validation data and test, we use $99\%$ of the \texttt{300K Random Images} dataset (due to its large size). For the ID data, we use $50\%$ for in-distribution training samples and $50\%$ for generating the mixture. 
 
We repeat each experiment 5 times with 5 separate seeds.

\section{Proof of Proposition \ref{prop:sigmoid_loss}}\label{sec:prop_proof}

In this Section, we prove Proposition \ref{prop:sigmoid_loss}. We begin by giving notation and proving an important Lemma.  To ease notation, we write $\P_\text{out} =: \P_1$ and $\P_\text{in} =: \P_0$. Define the sigmoid loss for the OOD task
\begin{align*}
    R_y(g_\theta) := \mathbb{E}_{\*x \sim \P_\text{out}}(\sigma(-g_\theta(\*x)) \cdot y).
\end{align*}
Define 
\begin{align*}
    R^*_1 & := \inf_\theta R_1(g_\theta) \\
    & \text{ s.t. } R_{0}(g_\theta) \leq \alpha.
\end{align*}
Define
\begin{align*}
    R_\text{wild} & := \mathbb{E}_{\*x \sim \Pwild}(\sigma(-g_\theta(\*x)) \\
    & = \pi R_1(g_\theta) + (1-\pi) \mathbb{E}_{\*x \sim \P_{0}}(\sigma(-g_\theta(\*x)) \\
    & = \pi R_1(g_\theta) + (1-\pi) (1-R_{0}(g_\theta))
\end{align*}
where we used the symmetry of the sigmoid function, that is, $\sigma(z) + \sigma(-z) = 1$ for $z \in \R$. Now, define
\begin{align*}
    R^*_\text{wild} := & \inf_\theta R_\text{wild} (g_\theta) \\
    & \text{ s.t. } R_{0}(g_\theta) \leq \alpha. 
\end{align*}

Next, we prove a key Lemma for our proof. This Lemma has a similar proof to Theorem 1 in \cite{blanchard2010semi}, which applies to the $0/1$ loss. We establish an analogous result for the sigmoid loss. The key observation is that the symmetry property of the sigmoid loss enables a similar proof. 

\begin{lemma}\label{lem:ssnd}
Suppose that there exists $\theta^*$ such that $R_1(g_\theta) = R^*_1(g_{\theta^*})$ and $R_0(g_{\theta^*}) = \alpha$. Then, 
\begin{align*}
    R_1(g_\theta) - R_1^* \leq \frac{1}{\pi}(R_\text{wild}(g_\theta) - R_1^* + (1-\pi)(R_0(g_\theta) - \alpha)).
\end{align*}
\end{lemma}

\begin{proof}
We begin by showing that for any $\theta$, $R_\text{wild}(g_\theta) = R^*_\text{wild}$ and $R_0(g_\theta) \leq \alpha$ if and only if $R_1(g_\theta) = R^*_1$ and $R_0(g_\theta) = \alpha$. 

$\Longrightarrow$: First, suppose $\theta$ satisfies $R_\text{wild}(g_\theta) = R^*_\text{wild}$ and $R_0(g_\theta) \leq \alpha$. Suppose that either $R_0(g_\theta) < \alpha$ or $R_1(g_\theta) > R_1^*$. By the assumption in the Proposition, there exists $\theta^*$ such that $R_1(g_\theta) = R^*_1(g_{\theta^*})$ and $R_0(g_{\theta^*}) = \alpha$. Then, we have that
\begin{align*}
    R^*_\text{wild}(g_{\theta^*}) & = \pi R_1(g_{\theta^*}) + (1-\pi) (1 - R_0(g_{\theta^*}) \\
    & + \pi R_1(g_{\theta^*}) + (1-\pi) (1 - \alpha) \\
    & < \pi R_1(g_{\theta}) + (1-\pi) (1 - R_0(g_{\theta}) \\
    & = R_\text{wild}(g_\theta).
\end{align*}
But, this contradicts the assumption that $R_\text{wild}(g_\theta) = R^*_\text{wild}$, completing this direction of the claim.

$\Longleftarrow$: Suppose $\theta$ satisfies $R_1(g_\theta) = R^*_1$ and $R_0(g_\theta) = \alpha$. By the assumption in the Lemma, there exists $\theta^*$ such that $R_1(g_\theta) = R^*_1(g_{\theta^*})$ and $R_0(g_{\theta^*}) = \alpha$. Towards a contradiction, suppose that $R_\text{wild}(g_{\theta^*}) < R_\text{wild}(g_{\theta})$. Then, using $\pi > 0$, we have that
\begin{align*}
    R_1(g_{\theta^*}) & = \frac{1}{\pi}(R_\text{wild}(g_{\theta^*}) - (1-\pi)(1-R_0(g_{\theta^*})) \\
    & < \frac{1}{\pi}(R_\text{wild}(g_{\theta}) - (1-\pi)(1-\alpha) \\
    & = R_1(g_\theta)
\end{align*}
but this contradicts our assumption on $g_\theta$. This completes the proof of the claim.

The established claim implies that $R^*_\text{wild} = \pi R^*_1 +(1-\pi)(1-\alpha)$. The result now follows by subtracting this equality from $R^*_\text{wild}(g_\theta) = \pi R_1(g_\theta) + (1-\pi) (1-R_0(g_\theta)) $. 

\end{proof}

Here we restate Proposition \ref{prop:sigmoid_loss} with all technical details included. Let $\Theta \subset \R^p$ where we have that $\theta \in \Theta$. Recall that for the purposes of this proposition we assume that $f_\theta(x) : \R^d \mapsto [0,1]$ and define $\L_\text{cls}(s, y)= y \log(\frac{1}{s}) + (1-y) \log(\frac{1}{1-s})$ for $s \in [0,1]$.  Define
\begin{align}
     \widehat{\theta}_\epsilon \longleftarrow \argmin_{\theta \in \Theta} & \frac{1}{m} \sum_{i=1}^{m} \sigma (-g_\theta(\tilde{\*x}_i)) \label{eq:theory_opt_prob} \\ 
    \text{s.t.} & \frac{1}{n} \sum_{j=1}^n \sigma(g_\theta({\*x}_j) \leq \alpha + \epsilon \nonumber \\
    &  \frac{1}{n} \sum_{j=1}^n \L_\text{cls}(f_\theta(\*x_j) , y_j)  \leq \tau + \epsilon \nonumber
\end{align}

Recall the optimization problem of interest:
\begin{align}
    \inf_{\theta \in \Theta} & ~ \mathbb{E}_{\*x \sim \P_\text{out}} \sigma (-g_\theta(\tilde{\*x}_i)) \label{eq:goal_optimization_problem_tractable_supp} \\
    \text{s.t. } &  \mathbb{E}_{\*x \sim \Pin}\sigma (g_\theta({\*x}_j)) \leq \alpha \nonumber \\
    %\Pin(g_{\theta}(\*x) = \text{out}) \leq \alpha  \nonumber \\
    & \mathbb{E}_{(\*x, y) \sim \P_{\mathcal{X}\mathcal{Y}}}[ \L_\text{cls}(f_\theta(\*x) , y)] \leq \tau. \nonumber
\end{align}
and let \texttt{opt} denote its value.

We will make the following mild assumption and describe settings where it holds later.

\begin{assumption}\label{assum:theory}
There exists $\theta^*$ such that $R_1(g_\theta) = R^*_1(g_{\theta^*})$, $R_0(g_{\theta^*}) = \alpha$, and $\mathbb{E}_{(\*x, y) \sim \P_{\mathcal{X}\mathcal{Y}}}[\L_\text{cls}(f_{\theta^*}(\*x) , y)] \leq \tau.$
\end{assumption}

\begin{proposition}\label{prop:sigmoid_loss_theory}
Suppose $K=2$. Suppose Assumption \eqref{assum:theory} holds.
Define $\epsilon_k :=\sqrt{\frac{2 \ln(6/\delta)}{k}} +2\underset{h \in \{f,g\}}{\max} \underset{\P \in \{\P_1, \P_\text{wild}\}}{\max} \underset{{\tilde{\*x}_{1},\ldots, \tilde{\*x}_{m} \sim \P}}{\E} \,\underset{{\eta_1, \ldots, \eta_m}}{\E} \underset{\theta \in \Theta}{\sup} \frac{1}{k} \sum_{i=1}^k \eta_i h_\theta(\tilde{\*x}_{i}) $ 
where $\eta_1,\ldots, \eta_k$ are \emph{i.i.d.} and $\P(\eta_i=1) = \P(\eta_i=-1)=1/2$. Let $\widehat{\theta}_\epsilon$ solve \eqref{eq:theory_opt_prob} with tolerance $\epsilon =c \epsilon_n$ where $c$ is a universal positive constant. Then, with probability at least $1-\delta$
\begin{enumerate}
\item $\E_\text{out}\sigma(-g_{\widehat{\theta}_\epsilon}(\*x)) \leq \texttt{opt} + c_1 \pi^{-1}(\epsilon_n + \epsilon_m)$,
    \item $\E_\text{in}\sigma(g_{\widehat{\theta}_\epsilon}(\*x))  \leq \alpha + c_2 \epsilon_n$, and
    \item $\E_{\mathcal{X}\mathcal{Y}} [\L_\text{cls}(f_{\widehat{\theta}_\epsilon}(\*x) , y) ]  \leq \tau + c_3 \epsilon_n $.
\end{enumerate}
\end{proposition} 

\begin{proof}
Define the following events
\begin{align*}
    \Sigma_1 & = \{ |\frac{1}{m} \sum_{i=1}^{m} \sigma (-g_\theta(\tilde{\*x}_i)) - R_\text{wild}(g_\theta)| \leq 2 \E_{\tilde{\*x}_{1},\ldots, \tilde{\*x}_{m} \sim \P_\text{wild}} \E_{\eta_1, \ldots, \eta_m} \sup_{\theta \in \Theta} \frac{1}{m} \sum_{i=1} \eta_i g_\theta(\tilde{\*x}_{i}) + c \sqrt{\frac{2 \ln(6/\delta)}{m}} : \forall \theta \in \Theta \} \\
            \Sigma_2 & = \{ |\frac{1}{n} \sum_{i=1}^{n} \sigma (g_\theta(\*x_i)) - R_0(g_\theta)| \leq 2 \E_{\*x_{1},\ldots, \*x_{n} \sim \P_0} \E_{\eta_1, \ldots, \eta_n} \sup_{\theta \in \Theta} \frac{1}{n} \sum_{i=1} \eta_i g_\theta(\*x_{i}) + c \sqrt{\frac{2 \ln(6/\delta)}{n}} : \forall \theta \in \Theta \} \\
            \Sigma_3 & = \{ |\frac{1}{n} \sum_{i=1}^{n} \L_\text{cls}( f_\theta(\*x_i),y_i) - \E_{\mathcal{X}\mathcal{Y}} [\L_\text{cls}(f_{\theta}(\*x) , y) ]| \leq 2 \E_{\*x_{1},\ldots, \*x_{n} \sim \P_0} \E_{\eta_1, \ldots, \eta_n} \sup_{\theta \in \Theta} \frac{1}{n} \sum_{i=1} \eta_i f_\theta(\*x_{i}) \\
            & + c \sqrt{\frac{2 \ln(6/\delta)}{n}} : \forall \theta \in \Theta \} 
\end{align*}
By Lemma \ref{lem:concentration}, we have that $\P(\Sigma_i) \geq 1- \delta/3$ for all $i=1,2,3$. Then, by the union bound, we have that $\Sigma : = \Sigma_1 \cap \Sigma_2 \cap \Sigma_3$ holds with probability at least $1-\delta$. Assume $\Sigma$ holds for the remainder of the proof. 

Note that we have
\begin{align}
    R_0(g_{\widehat{\theta}_\epsilon}) - R_0(g_{\theta^*}) & = R_0(g_{\widehat{\theta}_\epsilon}) - \frac{1}{n} \sum_{i=1}^{n} \sigma (-g_{\widehat{\theta}_\epsilon}(\*x_i)) \nonumber \\
    & + \frac{1}{n} \sum_{i=1}^{n} \sigma (-g_{\widehat{\theta}_\epsilon}(\*x_i)) - \frac{1}{n} \sum_{i=1}^{n} \sigma (-g_{\theta^*}(\*x_i)) \nonumber \\
    & + \frac{1}{n} \sum_{i=1}^{n} \sigma (-g_{\theta^*}(\*x_i)) - R_0(g_{\theta^*}) \nonumber \\
    & \leq R_0(g_{\widehat{\theta}_\epsilon}) - \frac{1}{n} \sum_{i=1}^{n} \sigma (-g_{\widehat{\theta}_\epsilon}(\*x_i)) \nonumber \\
    & + \frac{1}{n} \sum_{i=1}^{n} \sigma (-g_{\theta^*}(\*x_i)) - R_0(g_{\theta^*}) \label{eq:ineq_optimization_prob} \\
    & \leq 4 \E_{\tilde{\*x}_{1},\ldots, \tilde{\*x}_{m} \sim \P_\text{wild}} \E_{\eta_1, \ldots, \eta_m} \sup_{\theta \in \Theta} \frac{1}{m} \sum_{i=1} \sigma_i g_\theta(\tilde{\*x}_{i}) + 2c \sqrt{\frac{2 \ln(2/\delta)}{m}} \label{eq:apply_event_1}.
\end{align}
\eqref{eq:ineq_optimization_prob} follows since $\widehat{\theta}_\epsilon$ is feasible for the optimization problem \eqref{eq:theory_opt_prob} because, by the choice of $\epsilon $ and $\Sigma$, 
\begin{align*}
R_0(g_{\theta^*}) & \leq \alpha \\
\E_{\*x \sim \P_0} \max(1- f_{\theta^*}(\*x)y,0) &  \leq \tau 
\end{align*}
Therefore, by definition of $\widehat{\theta}_\epsilon$ as the minimizer of \eqref{eq:theory_opt_prob}, we have that
\begin{align*}
\frac{1}{n} \sum_{i=1}^{n} \sigma (-g_{\widehat{\theta}_\epsilon}(\*x_i)) - \frac{1}{n} \sum_{i=1}^{n} \sigma (-g_{\theta^*}(\*x_i))\leq 0.
\end{align*}
\eqref{eq:apply_event_1} follows by the event $\Sigma_1$.

Similarly,
\begin{align}
    \E_\text{in}\sigma(g_{\widehat{\theta}_\epsilon}(\*x)) & =  \frac{1}{n} \sum_{i=1}^{n} \sigma (g_\theta(\*x_i)) + \E_\text{in}\sigma(g_{\widehat{\theta}_\epsilon}(\*x)) - \frac{1}{n} \sum_{i=1}^{n} \sigma (g_\theta(\*x_i)) \nonumber \\
    & \leq \alpha + c \epsilon_n + \E_\text{in}\sigma(g_{\widehat{\theta}_\epsilon}(\*x)) - \frac{1}{n} \sum_{i=1}^{n} \sigma (g_\theta(\*x_i))  \label{eq:ineq_optimization_prob_2}\\
    & \leq \alpha + c_2 \epsilon_n \label{eq:apply_event_2}
\end{align}
where \eqref{eq:ineq_optimization_prob_2} follows from the definition of $\widehat{\theta}_\epsilon$ and \eqref{eq:apply_event_1} follows from $\Sigma_2$. This establishes claim 2 in the Proposition.

Claim 3 follows by a similar argument to claim 2. Finally, claim 1 follows by \eqref{eq:apply_event_1}, \eqref{eq:apply_event_2}, and  Lemma \ref{lem:ssnd}.

\end{proof}

\begin{lemma}\label{lem:concentration}
Let $\delta \in (0,1)$. Let $\*x_1,\ldots, \*x_k \sim \P$ and $y_i \sim \text{Bernoulli}(p(\*x_i))$ for $i \in [k]$. Let $\eta_i$ be a Rademacher random variable, i.e., $\P(\eta_i=1) = \P(\eta_i=-1)=1/2$. Then,
\begin{itemize}
\item With probability at least $1-\delta$, for all $\theta \in \Theta$
\begin{align*}
    |\frac{1}{k} \sum_{i=1}^{k} \sigma (-g_\theta(\*x_i)) - \E_{\*x \sim \P} \sigma (-g_\theta(\*x)) | \leq 2 \E_{\*x_{1},\ldots, \*x_{k} \sim \P} \E_{\eta_1, \ldots, \eta_m} \sup_{\theta \in \Theta} \frac{1}{k} \sum_{i=1}^k \eta_i g_\theta(\*x_i) + c \sqrt{\frac{2 \ln(2/\delta)}{k}}
\end{align*}
\item With probability at least $1-\delta$, for all $\theta \in \Theta$
\begin{align*}
    |\frac{1}{k} \sum_{i=1}^{k} \sigma (g_\theta(\*x_i)) - \E_{\*x \sim \P} \sigma (g_\theta(\*x)) | \leq 2 \E_{\*x_{1},\ldots, \*x_{k} \sim \P} \E_{\eta_1, \ldots, \eta_m} \sup_{\theta \in \Theta} \frac{1}{k} \sum_{i=1}^k \eta_i g_\theta(\*x_i) + c \sqrt{\frac{2 \ln(2/\delta)}{k}}
\end{align*}
\item With probability at least $1-\delta$, for all $\theta \in \Theta$
\begin{align*}
    |\frac{1}{n} \sum_{i=1}^{n} \L_\text{cls}( f_\theta(\*x_i),y_i) - \E_{\*x \sim \P, y \sim \text{Bernoulli}(p(\*x_i))} [\L_\text{cls}(f_{\theta}(\*x) , y) ]|  & \leq 2 \E_{\*x_{1},\ldots, \*x_{k} \sim \P} \E_{\eta_1, \ldots, \eta_n} \sup_{\theta \in \Theta} \frac{1}{k} \sum_{i=1} \eta_i f_\theta(\*x_{i}) \\
            & + c \sqrt{\frac{2 \ln(2/\delta)}{k}}.
\end{align*}
\end{itemize}
\end{lemma}

\begin{proof}
We show the first bullet point. The second and third bullet points follow by a similar argument. Using Mcdiarmid's inequality, we have that with probability at least $1-\delta$ for all $\theta \in \Theta$
\begin{align*}
    |\frac{1}{k} \sum_{i=1}^{k} \sigma (-g_\theta(\*x_i)) - \E_{\*x \sim \P} \sigma (-g_\theta(\*x)) | \leq 2 \E_{\*x_{1},\ldots, \*x_{k} \sim \P} \E_{\eta_1, \ldots, \eta_m} \sup_{\theta \in \Theta} \frac{1}{k} \sum_{i=1}^k \eta_i \sigma(g_\theta(\*x_i)) + c \sqrt{\frac{2 \ln(2/\delta)}{k}}
\end{align*}
Then, using the contraction Lemma and the fact that the sigmoid function $\sigma$ is $1$-Lipschitz, we have that
\begin{align*}
    \E_{\*x_{1},\ldots, \*x_{k} \sim \P} \E_{\eta_1, \ldots, \eta_m} \sup_{\theta \in \Theta} \frac{1}{k} \sum_{i=1}^k \eta_i \sigma(g_\theta(\*x_i)) \leq \E_{\*x_{1},\ldots, \*x_{k} \sim \P} \E_{\eta_1, \ldots, \eta_m} \sup_{\theta \in \Theta} \frac{1}{k} \sum_{i=1}^k \eta_i g_\theta(\*x_i).
\end{align*}
The result follows by combining the above two inequalities.
\end{proof}

As an example where Assumption \ref{assum:theory} holds, consider for instance when $\theta = \begin{pmatrix}
w_1 \\
w_2
\end{pmatrix}$ with $w_1, w_2 \in \R^d$ and $g_\theta(x) = w_1^\t x$ and $f_\theta(x) = w_2^\t x$. Then if $\P_0$ is absolutely continuous wrt the Lebesgue measure, Assumption \ref{assum:theory} holds. We could similar replace the linear maps $g$ and $f$ with neural networks that share a penultimate layer. See \citet{blanchard2010semi} for a more detailed discussion and for more examples.

\section{Validation using $\P_\text{wild}$ data}
\label{sec:validation}
In this Section, we discuss how to use data from $\P_\text{wild}$ for a validation procedure and demonstrate its feasibility. For simplicity, we focus on the OOD task since it is standard to have a clean ID validation set for classification and therefore this captures the main difficulty. To ease notation, we write $\P_\text{out} = \P_1$ and $\P_\text{in} = \P_1$. Here, overloading notation, we assume access to a holdout set from $\P_0$ and the mixture $\P_\text{wild}$:
\begin{itemize}
    \item $\*x_1,\ldots, \*x_n \sim \P_0$
    \item $\tilde{\*x}_{1},\ldots, \tilde{\*x}_{m} \sim \P_\text{wild} := (1-\pi) \P_0 + \pi \P_1$ ($\pi \in (0,1]$ unknown)
\end{itemize}
We suppose that we have access to a small, finite set of models $\G \subset \{ g : \R^d \mapsto \R \}$. $\G$ is typically obtained from training a model with a set of distinct hyperparameters, generating one $g \in \G$ for each hyperparameter configuration. Note that $\G$ is totally generic. As is typical in the OOD literature, we obtain from $\G$ a set of OOD predictors by thresholding each $g \in \G$ as follows:
\begin{align*}
    \H & := \{ \sign (g(x) - \tau ) : g \in \G, \tau \in \R\}.
\end{align*}

We now introduce some notation, overloading notation from Section \ref{sec:prop_proof}. Define
\begin{align*}
    R_y(g_\theta) := \mathbb{E}_{\*x \sim \P_\text{out}}(\d1 
    \{h(\*x)) \neq y\}.
\end{align*}
Define the optimization problem
\begin{align*}
    R^*_1 & := \inf_{h \in \H} R_1(h) \\
    & \text{ s.t.} R_0(h) \leq \alpha.
\end{align*}
Define risk for $\P_\text{wild}$
\begin{align*}
    R_\text{wild} & := \mathbb{E}_{\*x \sim \Pwild}(\d1 
    \{h(\*x)) \neq 1\}) \\
    & = \pi R_1(h) + (1-\pi) (1-R_0(h))
\end{align*}
 Now, define another similar optimization, only changing the objective:
\begin{align*}
    R^*_\text{wild} & := \inf_\theta R_\text{wild} (h) \\
    & \text{ s.t. } R_0(h) \leq \alpha. 
\end{align*}

We choose the $\widehat{h} \in \H$ that minimizes the FNR@95 on the holdout set:
\begin{align*}
    \widehat{h}_\epsilon \in & \argmin_{h \in \H} \frac{1}{m} \sum_{i=1}^m \d1\{h(\tilde{\*x}_i) \neq 1\}  \\
    & \text{ s.t. } \frac{1}{n} \sum_{i=1}^n \d1\{h(\*x_i) \neq -1\}  \leq \alpha + \epsilon
\end{align*}
where we write $\widehat{h} := \widehat{h}_0$. We emphasize that this procedure is not only intuitive; it is also justified theoretically by applying Theorem 2 from \cite{blanchard2010semi}. Theorem 2 requires that the following condition is satisfied:

\begin{assumption}\label{assum:validation}
For any $\alpha \in (0,1)$, there exists $h^* \in \G$ such that $R_0(h^*) = \alpha$ and $R_1(h^*) = R_{1,\alpha}^*(\G)$.
\end{assumption}

Here, we show that $\H$ satisfies Assumption \ref{assum:validation} if $P_0$ is absolutely continuous with respect to the Lebesgue measure. Fix some $g \in \G$ and define $h_\tau(x) := \d 1\{h(x) > \tau\}$. Notice that if $\tau > \tau^\prime$, then
\begin{align*}
    h_\tau(x) \leq  h_{\tau^\prime}(x).
\end{align*}
Thus, as discussed in the Remark of page 2978 in \cite{blanchard2010semi}, we have that if for a given $\tau$ we have that $R_0(h_\tau) < \alpha$, using the absolute continuity of $P_0$, we can find a $\tau^\prime$ such that $R_0(h_{\tau^\prime}) = \alpha$ and $R_1(h_{\tau^\prime}) \leq R_1(h_{\tau})$. Since this holds for any $g \in \G$, this implies that Assumption \ref{assum:validation} holds. Then, as a Corollary from Theorem 2 of \cite{blanchard2010semi}, we obtain
\begin{corollary}
Let $\epsilon_k := \sqrt{\frac{\mathcal{VC}(\H) - \log(\delta)}{k}}$ where $\mathcal{VC}(\H)$ denotes the VC dimension of $\H$. If $\epsilon = c \epsilon_n$, with probability at least $1-\delta$
\begin{enumerate}
    \item $R_{1}(\widehat{h}_\epsilon) \leq R^*_1 + c \pi^{-1}( \epsilon_n + \epsilon_m)$, and
    \item $R_{0}(\widehat{h}_\epsilon) \leq \alpha + c \pi^{-1} \epsilon_n$.
\end{enumerate}
\end{corollary}

%%%%%%%%%%%%%%%%%%%%%%%%%%%%%%%%%%%%%%%%%%%%%%%%%%%%%%%%%%%%%%%%%%%%%%%%%%%%%%%
%%%%%%%%%%%%%%%%%%%%%%%%%%%%%%%%%%%%%%%%%%%%%%%%%%%%%%%%%%%%%%%%%%%%%%%%%%%%%%%

\end{document}